\documentclass[preprint]{imsart}
\usepackage[top=2.8cm, bottom=2.8cm, left=2.8cm, right=2.8cm]{geometry}

\usepackage[utf8]{inputenc}
\usepackage{xr}

\usepackage{tikz}

\usepackage{amsthm,amsmath,amssymb}
\usepackage{parskip}
\usepackage[authoryear]{natbib}
\usepackage{hypernat}
\usepackage{marvosym}
\usepackage[hang,small,bf]{caption}
\usepackage{mathtools}

\usepackage{hyperref}
\hypersetup{
    colorlinks,%
    citecolor=blue,%
    filecolor=black,%
    linkcolor=blue,%
    urlcolor=blue
}
\usepackage{hypernat}
\usepackage[none]{hyphenat}
\usepackage{enumerate}

\usepackage{color}
\usepackage{bbm}
\usepackage{thmtools} 
\usepackage{xcolor}
\usepackage{subcaption}
\usepackage{textcomp}
\usepackage{graphicx}
\usepackage{float}
\usepackage{dirtytalk}
\usepackage{mathtools}
\usepackage{algpseudocode}
\usepackage{algorithm}

\usepackage{tikz}
\usepackage{mathtools}
\mathtoolsset{showonlyrefs}

\newtheoremstyle{exampstyle}
{8pt} 
{8pt} 
{\it} 
{} 
{\bfseries} 
{.} 
{.5em} 
{} 

\theoremstyle{exampstyle}

\newtheorem{theorem}{Theorem}[section]
\newtheorem{proposition}[subsection]{Proposition}
\newtheorem{lemma}{Lemma}
\newtheorem{corollary}[theorem]{Corollary}

\startlocaldefs

\newcommand{\eat}[1]{}

\DeclareMathOperator*{\argmax}{\arg\hspace{3pt}max}

\renewcommand{\hat}[1]{\widehat{#1}}
\renewcommand{\tilde}[1]{\widetilde{#1}}

\theoremstyle{plain}

\def\beq{\begin{equation}}
\def\eeq{\end{equation}}
\def\ba{\begin{enumerate}[(a)]}
\def\bei{\begin{enumerate}[(i)]}
\def\be{\begin{enumerate}[(1)]}
\def\ee{\end{enumerate}}
\def\bi{\begin{itemize}}
\def\ei{\end{itemize}}
\def\beg{\begin{eg}}
\def\eeg{\end{eg}}
\def\bd{\begin{defn}}
\def\ed{\end{defn}}
\def\bt{\begin{thm}}
\def\et{\end{thm}}
\def\bl{\begin{lemma}}
\def\el{\end{lemma}}
\def\bfac{\begin{fact}}
\def\efac{\end{fact}}

\def\bc{\begin{corollary}}
\def\ec{\end{corollary}}
\def\bp{\begin{prop}}
\def\ep{\end{prop}}
\def\bo{\begin{observe}}
\def\eo{\end{observe}}
\def\bas{\begin{assumption}}
\def\eas{\end{assumption}}

\endlocaldefs

\begin{document}

\begin{frontmatter}
\title{Convex Smoothed Autoencoder-Optimal Transport model}

\runtitle{Convex Smoothed Autoencoder-Optimal Transport model}

\begin{aug}
\author{\fnms{Aratrika} \snm{Mustafi}\ead[label=e1]{am5322@columbia.edu}}
\affiliation{
Department of Statistics, Columbia University
}
\runauthor{Mustafi}

\address{Department of Statistics, Columbia University \\
\printead{e1} \\
\vspace{10pt}
}
\end{aug}

\begin{abstract}
Generative modelling is a key tool in unsupervised machine learning which has achieved stellar success in recent years. Despite this huge success, even the best generative models such as Generative Adversarial Networks (GANs) and Variational Autoencoders (VAEs) come with their own shortcomings, mode collapse and mode mixture being the two most prominent problems. In this paper we develop a new generative model capable of generating samples which resemble the observed data, and is free from mode collapse and mode mixture. Our model is inspired by the recently proposed Autoencoder-Optimal Transport (AE-OT) model (\citet{AEOT}) and tries to improve on it by addressing the problems faced by the AE-OT model itself, specifically with respect to the sample generation algorithm. Theoretical results concerning the bound on the error in approximating the non-smooth Brenier potential by its smoothed estimate, and approximating the discontinuous optimal transport map by a smoothed optimal transport map estimate have also been established in this paper. 
\end{abstract}



\end{frontmatter}

\section{Introduction}\label{sec:Intro}
The success of generative models in recent years has caused a paradigm shift in the field of machine learning. Generative modelling is one of the most important types of unsupervised learning, with recent applications in semi-supervised learning as well. It addresses the problem of probability density estimation, which is a core problem in unsupervised learning. Given training data, the primary goal of generative models is to generate new samples or observations having the same or approximately the same distribution as the training data. There are several different categories of generative models, each dealing with a different flavor of density estimation. Some generative models deal with explicit and tractable exact density estimation, such as fully visible belief networks (\citet{Frey}, \citet{Frey1998GraphicalMF}) and nonlinear independent components analysis (\citet{deco1995higher}, \citet{DinhNICE}, \citet{DinhRealNVP}). Some other models deal with explicit but approximate density estimation, such as Variational Autoencoders (\citet{KingmaVAEbook}, \citet{VAETutorial}, \citet{kingma2013fast}, \citet{rezende2014}, \citet{kingmaimproveVar}, \citet{chen2016variational}), Boltzmann Machines (\citet{MassParFahlman}, \citet{AckleyBoltz}, \citet{Hinton-84}, \citet{HintonSejnowski86}) and deep Boltzmann machines (\citet{salakhutdinovDeepBoltzmann}). Finally, some generative models are concerned with implicit density estimation, which are capable of sampling from the estimated probability density without explicitly estimating it. These encompass generative stochastic networks (\citet{bengio2014deep}) and, perhaps the most popular and widely successful generative model in recent years, Generative Adversarial Network (\citet{GoodfellowGAN}, \citet{goodfellow2016nips}, \citet{Radford2016UnsupervisedRL}, \citet{WGAN}, \citet{ImprovedWGANS}, \citet{StyleGAN}, \citet{PacGAN}, \citet{CycleGAN}, \citet{Pix2PixGAN}, \citet{StackGAN}).

In spite of the tremendous amount of success achieved by these generative models, in particular Variational Autoencoders (VAEs) and Generative Adversarial Networks (GANs), they are found to suffer from a few drawbacks. Most important among these drawbacks are mode collapse in GANs and mode mixtures in VAEs. Mode collapse is said to occur when the target distribution of samples is multimodal, but the sample generation procedure fails to produce any sample from one or more modal regions. For example, the MNIST dataset (\citet{MNISTdata}) contains black and white images of handwritten digits from 0 to 9, constituting 10 distinct classes or categories of observations. It is reasonable to believe that the distribution of these images will have 10 distinct modes corresponding to each category of images. When the generative model fails to produce samples corresponding to any particular category (say there are no samples containing the digit 6), an extreme form of mode collapse is said to occur. A slightly weaker form of mode collapse occurs when the proportion of generated samples corresponding to a particular mode is much smaller than the proportion of samples corresponding to the same mode in the observed data. On the other hand, mode mixture is said to occur when the target distribution has its support on a manifold with well-separated modal regions, but the generated samples lie in between these modal regions, corresponding to low probability zones of the target distribution. In most cases, such samples combine characteristics of samples belonging to the separate modes between which they lie, and thus are quite different from the observed data. In the case of MNIST data, a generated sample which looks like a combination of a 5 and 6 is a mixture between the two modes of the target distribution corresponding to the digits 5 and 6.

Recently, these shortcomings have been addressed using the theory of optimal transport in the paper \citet{modecollapsereg}. The generator function of a GAN can be viewed as a composition of an optimal transport map (with the noise distribution as source and a conceptual \say{latent code} distribution) with a decoder neural network. It is observed that this optimal transport map is discontinuous and it leads to the discontinuity of the generator function in GANs. This makes the generator function unfit for modelling using neural networks. Forcefully modelling such a discontinuous function using neural networks creates the problem of mode collapse in GANs. We elaborate on this optimal transport perspective of  GANs and its implications in Section \ref{subsec:GAN and the role of optimal transport theory in generative modelling}.

This viewpoint of GANs led to the development of a new generative model, Autoencoder-Optimal Transport (AE-OT) model, proposed in \citet{AEOT}. The AE-OT model comprises of an autoencoder. The encoder network of the autoencoder creates an empirical latent code distribution corresponding to the training data in a latent space, with the latent codes representing the essential features of the observations. An optimal transport map between a noise distribution and the empirical latent code distribution is computed, and then a continuous linear approximation of the optimal transport map is constructed. This continuous function coupled with the decoder neural network of the autoencoder serves the role of the generator function in this model. The model is described in detail in Section \ref{subsec:The AE-OT model}.

However, the AE-OT methodology also suffers from a few drawbacks of its own, and we attempt to understand and illustrate them. The optimal transport map between the noise distribution and the empirical latent code distribution is discontinuous, and maps every possible sample generated from the noise distribution to one of the latent codes corresponding to observed data, which in turn gets mapped to a sample exactly equal to an observed sample, if the autoencoder is trained sufficiently. To generate new samples similar to the observed data without exact reconstruction, the optimal transport map needs to be smoothed and made globally continuous. In the AE-OT model, this is achieved by a piecewise linear extension of the OT map, with the extended map having as its domain a simplicial complex obtained by triangulating the latent codes corresponding to observed data. Then, by a complicated procedure depending upon a user-specified parameter which is difficult to interpret and tune, the regions of discontinuity of the estimated optimal transport map, known as singularity sets, are estimated, and samples from the noise distribution which get mapped to singularity sets are rejected, since these samples are mixtures of modes of the distribution of the observed data. Thus, the AE-OT methodology involves a complicated, unintuitive and computationally expensive method of generating new samples based on triangulations, and wastes computational resources in generating a large number of potential samples which are ultimately discarded by the rejection sampling scheme employed within this methodology. The technical details regarding the triangulation of the latent codes, construction of the piecewise linear extension $\tilde{T}$ and singular set detection are described in detail in Section \ref{subsec:The AE-OT model}, and even more elaborately in \citet{AEOT}.

The main motivation behind the paper \citet{AEOT} is to tackle mode collapse and mode mixture problems in general generative models, not only for GANs, by providing a theoretical justification for these issues and developing a generative model capable of mitigating them. In this paper, we proceed one step further by addressing the drawbacks of AE-OT. We develop a generative model which modifies the generative module of AE-OT in order to improve the sample generation methods followed in \citet{AEOT}, based on ideas of convex smoothing proposed in \citet{nesterov1998introductory} and \citet{Convex}.

Our primary contributions in this paper are:  

\begin{itemize}
    \item We develop a generative model which produces good quality samples, in the sense that they resemble the observed data and do not suffer from mode collapse and mode mixture.
    \item We provide a theoretical validation for the efficacy of the convex smoothed AE-OT model by proving an uniform bound on the error of approximation of the optimal transport map between the noise distribution and the empirical latent code distribution, which serves as a measure of how closely the generated samples resemble the observed samples.
    \item We improve upon the sample generation method of the AE-OT model while developing our model by removing the need for rejection sampling, thus saving precious computational time and resources.
    \item In contrast to the method of latent vector generation in AE-OT, which ultimately produces linear combinations of encoded latent vectors corresponding to training data, our proposed method is not restricted to producing only linear combinations of latent vectors and potentially allows one to cover the entire manifold support of the distribution of latent vectors defined within the latent space corresponding to the autoencoder used.
    \item Our model is dependent upon an user specified parameter controlling the degree of accuracy of our method, ensuring the mitigation of mode-collapse and mode-collapse without exact reconstruction of the training data, and having the additional benefit of being more interpretable and easier to choose than the tuning parameter $\theta$, used in the AE-OT model for controlling the degree of mode-mixture.
    \item We propose a strategy for choosing the optimal value of the user-specified parameter based on a two-sample statistical test of equality of the distribution from which samples are generated and the true distribution of the data we intend to generate, based on the generated  and observed samples. We show that there is a trade-off between diversity in the generated samples and the degree of similarity between the generated samples and training samples, and our proposed strategy provides an optimal balance between the two extreme scenarios.
    
\end{itemize}

The organization of the paper is as follows. Section \ref{sec:Intro} provides a brief introduction to generative modelling along with existing models in the literature, with particular focus on the Autoencoder-Optimal Transport model (\citet{AEOT}) and an overview of our contributions in this paper. Section \ref{sec:AE-OT to modified AE-OT Framework} begins with a primer on Generative adversarial networks (GANs) along with the relevant elements of optimal transport theory. We then proceed to discuss the problems faced by GANs, providing the motivation for the development of the AE-OT model. The AE-OT model is discussed next along with its drawbacks. We then discuss our novel contribution in the form of a sample generation procedure based on the idea of convex smoothing (\citet{nesterov1998introductory}, \citet{Convex}) as an alternative to the sample generation method of the AE-OT model, and propose the convex smoothed AE-OT model, along with relevant theoretical justifications. Section \ref{sec:Algorithm} provides the complete algorithm for constructing the convex smoothed AE-OT model and obtaining the generated samples based on input training data. Section \ref{sec:Theoretical validation of the Convex Smoothed AE-OT model} contains a theoretical proof validating the use of the convex smoothed AE-OT model. Section \ref{sec:Experimental results} provides the experimental results obtained on applying the convex smoothed AE-OT model to simulated 2 dimensional multimodal datasets. Section \ref{sec:Conclusion} includes a concluding discussion.

\section{AE-OT to Convex Smoothed AE-OT Framework}\label{sec:AE-OT to modified AE-OT Framework}

In the family of generative models, GANs are among the most successful, being able to generate highly realistic samples, especially in case of image data, and hence they serve as a reference model for sample generation problem. Recently, the theory of optimal transport has been used to provide us a deeper insight into the GAN paradigm.

In the first two subsections, we will discuss the GAN paradigm from the optimal transport viewpoint and the difficulties faced by GANs, which provide us the motivation for developing the AE-OT model. Later, we propose our modification of the sample generation method, along with the complete sample generation algorithm as well as a procedure to choose the optimal level of approximation error to allow. 

\subsection{GAN and the role of optimal transport theory in generative modelling}\label{subsec:GAN and the role of optimal transport theory in generative modelling}

 Let $o_{1},o_{2},\dots,o_{n}$ denote $n$ observed data points (usually images) belonging to the image space or ambient space $\mathcal{X}$ and let $\eta$ be the the corresponding empirical distribution (true distribution). The manifold distribution hypothesis allows us to imagine a manifold $\Sigma$ within $\mathcal{X}$ on which the data/images reside, and $\eta$ is defined on the manifold support $\Sigma$.

A Generative Adversarial Network (GAN) consists of two components: A generator and a discriminator. We assume $x_{1},x_{2},\dots,x_{N}\sim \mu$ are samples generated from a tractable noise distribution $\mu$ (usually uniform or Gaussian) defined on a low-dimensional space $\mathcal{Z}$ (a latent space encoding latent features or essential characteristics of observed images). The generator neural network G, represented as a function $g_{\gamma}:\mathcal{Z} \rightarrow \mathcal{X}$, of a GAN transforms $x_{i}$'s into $g_{\gamma}(x_{i})$'s in $\mathcal{X}$ to generate new image samples having distribution $\zeta_{\gamma}$ i.e. $$x \sim \mu \rightarrow g_{\gamma}(x) \sim \zeta_{\gamma} \textrm{ where } x \in \mathcal{Z} \textrm{ and } g_{\gamma}(x) \in \mathcal{X} $$
Here $\gamma$ represents the neural network parameters corresponding to the generator network, and hence is used to parametrize both the generator function $g_{\gamma}$ and the empirical distribution of the generated samples $\zeta_{\gamma}$.
 The discriminator neural network D works as an adversary and attempts to discriminate between the generated image distribution $\zeta_{\gamma}$ and the true image distribution $\eta$, helping the generator network to learn from the training data. The Jenson Shannon divergence $JS(\eta ||\zeta_{\gamma})$ is used by discriminators in traditional GANs to measure the degree of dissimilarity between the two distributions (\citet{GoodfellowGAN}), while discriminators in Wasserstein GANs use the Wasserstein distance based on $L_{p}$ loss $W_{p}(\eta,\zeta_{\gamma})$ (\citet{arjovsky2017wasserstein}). The paper (\citet{geometricunderstanding}, \citet{geometricview}) shows that GANs try to learn the manifold $\Sigma$, together with the optimal transport map T between $\mu$ and $\eta$ using quadratic loss and a manifold parametrization $g$ which maps local coordinates in the latent space $\mathcal{Z}$ to the manifold $\Sigma$ within $\mathcal{X}$. 
 
 Following the papers (\citet{geometricunderstanding}, \citet{geometricview}), the GAN model can be understood, in principle, to accomplish two major tasks:
 \begin{enumerate}
     \item manifold learning, discovering the manifold structure of the data
     \item probability transformation, transforming a white noise to the data distribution.
 \end{enumerate}
 
 Accordingly, the generator map $g_{\gamma}:(\mathcal{Z}, \mu) \rightarrow\left(\Sigma, \zeta_{\gamma}\right)$ can be further decomposed into two steps,
$$
g_{\gamma}:(\mathcal{Z}, \mu) \stackrel{T}{\rightarrow}(\mathcal{Z}, \rho) \stackrel{g}{\rightarrow}\left(\Sigma, \zeta_{\gamma}\right)
$$
where $T$ is a transportation map, maps the white noise $\mu$ to $\rho$ in the latent space $\mathcal{Z}, g$ is the manifold parametrization, maps local coordinates in the latent space to the manifold $\Sigma$. Specifically, $g$ gives a local chart of the data manifold $\Sigma$, $\rho=g^{-1}_{\#}\eta$ is determined by the real data distribution $\eta$ and the encoding map $g^{-1}$ and $T$ realizes the probability measure transformation. Hence the generator $g_{\gamma}$ is equal to $g \circ T$. The goal of the GAN model is to find $g_{\gamma},$ such that the generated distribution $\zeta_{\gamma}$ fits the real data distribution $\eta,$ namely
$$
{g_{\gamma}}_{\#} \mu=\eta
$$
 
 
 Let $c: \mathcal{Z} \times \mathcal{Z} \rightarrow[0, \infty]$ be a measurable loss function: $c(x, y)$ represents the cost of transporting $x$ to $y$ where $x,y \in \mathcal{Z}$. For example, when $\mathcal{Z}=\mathbb{R}^{d},$ we can take $c$ to be the quadratic (or $L_{2}$ ) loss function
$$
c(x, y)=\|x-y\|^{2}
$$
The goal of optimal transport (Monge's problem) is to find a measurable transport map $T: \mathcal{X} \rightarrow \mathcal{Y}$ solving the (constrained) minimization problem
$$
\inf _{T} \int_{\mathcal{X}} c(x, T(x)) d \mu(x) \quad \text { subject to } \quad {T_{\#}}\mu=\rho
$$
where the minimization is over $T$ (a transport map), a measurable map from $\mathcal{Z}$ to $\mathcal{X}$, and ${T_{\#}}\mu$ is the push forward of $\mu$ by $T,$ i.e.,
$$
{T_{\#}}\mu(B)=\mu\left(T^{-1}(B)\right), \quad \text { for all } B \in \mathcal{X}.
$$

 

\subsection{Problems with GANs: Motivation behind the AE-OT model}\label{subsec:Problems with GANs: Motivation behind the AE-OT model}

GAN training is tricky, unstable and sensitive to hyperparameters. More importantly, they suffer from mode collapse where they learn to generate samples from a subset of modes from among the entire collection of modes in the true data distribution $\eta$. Mode collapse is said to occur also when proportions of generated samples from different modes do not match with the proportions of images belonging to the different modes in $\eta$. In addition, mode mixture may also occur when generated samples fall outside the true data manifold $\Sigma$ in between modal regions. 

\citet{modecollapsereg} discusses the following theoretical reasons behind mode collapse and mode mixture. Brenier Theory gives us the following result:

\begin{theorem}\label{thm:Brenier}
Suppose $\mathcal{X}$ and $\mathcal{Y}$ are the Euclidean space $\mathbb{R}^{d}$ and the transportation cost is the quadratic Euclidean distance $c(x, y)=\|x-y\|^{2}$ for every $x \in \mathcal{X},y \in \mathcal{Y}$. Furthermore $\mu$ is absolutely continuous, and both $\mu$ and $\rho$ have finite second order moments, $\int_{\mathcal{X}}|x|^{2} d \mu(x)+\int_{\mathcal{Y}}|y|^{2} d \rho(y)<\infty,$ then there exists a convex function $u: X \rightarrow \mathbb{R},$ the so-called Brenier potential, its
gradient map $T=\nabla u$ gives the solution to the Monge's problem,
$$
{T_{\#}} \mu=\rho
$$
The Brenier potential is unique upto a constant, hence the optimal transportation map is unique.
\end{theorem}

In case of GANs, we have $\mathcal{X}=\mathcal{Y}=\mathcal{Z}$. Further, since $\rho$ is a discrete distribution and $\mu$ is absolutely continuous, discrete Brenier theory under quadratic transportation cost can be used to show the existence of a convex piecewise linear continuous function $u$, the gradient of which is the optimal transport map $T = \nabla u$. Following Section 4 of \citet{modecollapsereg}, we can view this Brenier potential map $u$ geometrically as the upper envelope of a collection of hyperplanes (\citet{geometricview}) and $u$ can be parametrized uniquely upto an additive constant by a parameter $h$, referred to as the height vector (\citet{VarMinkowski}, \citet{AEOT}). $u$ is often referred to as $u_{h}$ using this parametrization. If $y_{1},y_{2},\dots,y_{n}$ are the latent codes obtained from the observed data $o_{1},o_{2},\dots,o_{n}$ using the inverse decoding function i.e. $g^{-1}(o_{i})=y_{i},i=1,2,\dots,n$, then $\rho=\frac{1}{n}\sum_{i=1}^{n}\delta_{y_{i}}$. In such a case, $u_{h}$ can be shown to be of the form
$$
u_{h}(x) = \max _{i=1}^{n}\left\{\pi_{h, i}(x)\right\}=\max _{i=1}^{n} \{{x^{T}}y_{i}+h_{i}\}
$$
where $\pi_{h, i}(x)={x^{T}}y_{i}+h_{i}$ is the hyperplane corresponding to $y_{i}$. Regularity theory of optimal transport given by Cafarelli and Figalli (\citet{modecollapsereg}) states that whenever the support of $\rho$ is non-convex or composed of disconnected components due to multimodality of $\rho$ (induced by multimodality of $\eta$), T is a discontinuous function, and hence $g_{\gamma}$ is discontinuous. Deep Neural Networks (DNN) can only model/approximate continuous functions and $g_{\gamma}$ lies outside the functional space represented using DNNs. This leads to the problems of unstable training, non-convergence of the training process, mode collapse and mode mixture. The regions where the transport map $T = \nabla u$ is discontinuous are referred to as singular sets, which are collection of points where $u$ has a non-unique sub-gradient. Singular sets are characterized by sharp ridges, indicated by large dihedral angles between adjoining hyperplanes. Latent vectors belonging to or close to singular sets correspond to mixtures between modes of the empirical latent code distribution $\rho$, which in turn correspond to images which are mixtures between modes of the distribution observed image distribution $\eta$. 


The AE-OT model (\citet{AEOT}) is motivated by these insights. It is proposed as a generative model free from the problems faced by GANs, yet able to generate new images which respect the diversity present in the real-life images and look realistic. 

\subsection{The AE-OT model}\label{subsec:The AE-OT model}

The AE-OT model has two major components (Fig. \ref{AE-OTmodel}) :

\begin{figure}[H]
\centering
\includegraphics[width=0.8\textwidth]{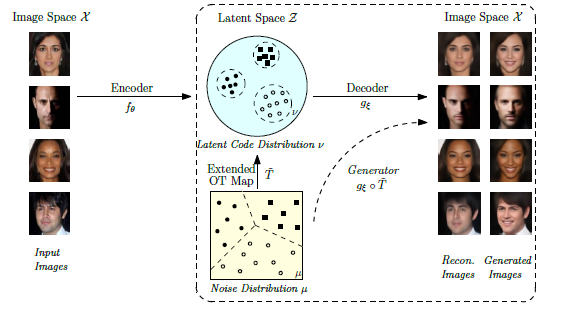}
\caption{AE-OT model (\citet{AEOT})}
\label{AE-OTmodel}
\end{figure}

\begin{enumerate}
    \item[i.] \textbf{Autoencoder} (AE) - An autoencoder is used for learning the data manifold $\Sigma$ in the image space $\mathcal{X}$. It learns the essential features of the data through dimensionality reduction. An autoencoder is composed two parts:
    \begin{itemize}
     \item[a.]An encoder network $\left(f_{\theta}\right)$ which encodes the data manifold from the image space $\mathcal{X}$ to the low-dimensional latent space $\mathcal{Z}$, and map the data distribution $\eta$ to the latent code distribution $\nu$ i.e.
     $$f_{\theta}:\mathcal{X}\rightarrow\mathcal{Z}$$ where $\theta$ represents the neural network parameters corresponding to the encoder network of the autoencoder.
     Both $\eta$ and $\nu$ are empirical discrete distributions of the form
     $$\eta=\frac{1}{n}\sum_{k=1}^{n} \delta_{o_{k}}\textrm{ and }\nu=\frac{1}{n}\sum_{k=1}^{n}\delta_{y_{k}}$$
     where $o_{k}$ is the $k$-th observed data point with $y_{k}=f_{\theta}(o_{k})\in\mathcal{Z}$ being the latent representation of $o_{k}, k=1,2,\dots,n$, and $\delta$ is the Dirac function.
     Here we note that the empirical latent code distribution $\rho$ is defined implicitly in case of GANs (since there is no explicit encoding network involved), while the empirical latent code distribution $\nu$ is explicitly defined based on the observed data as well as the explicit encoding network.
     \item[b.] A decoder network $g_{\xi}$ which maps/decodes the latent codes from $\mathcal{Z},$ back to the ambient/image space $\mathcal{X}$. i.e.
     $$g_{\xi}:\mathcal{Z}\rightarrow\mathcal{X}$$ where $\xi$ represents the neural network parameters corresponding to the decoder network of the autoencoder.
    \end{itemize}
    The encoded latent vectors/representations corresponding to the observed data is interpreted as essential features extracted from the data through dimensionality reduction, by minimizing the reconstruction loss between real images and reconstructed images obtained by passing the real images through the autoencoder. We refer to the review of autoencoders in \citet{geometricunderstanding}, the references therein and \citet{FirstAE},\citet{AEreview} for a detailed exposition on autoencoders.
    
    \item[ii.] \textbf{Optimal transport map} (OT) - New generated images can be obtained by the following steps : 
    \begin{enumerate}
        \item[a.] Generate random samples $$x\sim\mu$$ where $\mu$ is a tractable absolutely continuous noise distribution defined on $\Omega \subset \mathcal{Z}$ (say, uniform or Gaussian).
        \item[b.] Compute a probability distribution transformation between $\mu$ and the empirical latent code distribution $\nu$ , which is exactly the semi-discrete optimal transport map T under quadratic loss with $\mu$ as the source distribution and $\nu$ as the target distribution i.e.
        $${T_{\#}}\mu=\nu$$
        The Brenier potential $u$ corresponding to T can be parametrized uniquely by a \say{height} vector $h$ under a linear restriction, and can be referred to as $u_{h}$. $u_{h}$ is found by a convex optimization process involving Monte-Carlo simulation according to \citet{VarMinkowski} such that $T=\nabla u_{h}$. This is implemented using Algorithm (\ref{Alg.1}) in Section \ref{sec:Algorithm}. AE-OT tries to model the continuous Brenier potential map $u_{h}$ instead of the discontinuous OT map T using deep neural networks and thus potentially avoids the problems that GANs face.
        \item[c.] Smooth the optimal transport map T to obtain a continuous map $\tilde{T}$ by extending T to a globally continuous function $\tilde{T}$. The transport map T is piecewise linearly extended to a global continuous map $\tilde{T}$, where the image domain becomes a simplicial complex obtained by triangulating the latent codes $y_{1},y_{2},\dots,y_{n}$. 
        
        The technical details regarding the triangulation of the latent codes, construction of the piecewise linear extension $\tilde{T}$ and singular set detection are described in \citet{AEOT} and the reader is strongly advised to refer to it. This construction ensures that mode collapse cannot occur. We will discuss some of these details shortly.
        \item[d.] Define
        $$\Omega_{k}\left(u\right)\coloneqq\left\{x\in\Omega\subset\mathcal{Z}\mid \dim\left(\partial u(x)\right)=k\right\},k=0,1,2,\dots,\dim(\mathcal{Z})$$ where
        $\partial u(x)$ is the collection of sub-gradients of $u$ evaluated at $x$.
        Then the singularity set is $$\Omega_{sing}(u)={\bigcup_{k>0}}\Omega_{k}(u)$$
        which is essentially the region of discontinuity of the Optimal Transport map T. Detect the singularity set $\Omega_{sing}(u)$ in the source domain $\Omega \subset \mathcal{Z}$ of T. If $x\in\Omega_{sing}(u)$, then $\tilde{T}(x)$ represents a sample which is a mixture between modes of the distribution $\nu$, and consequently $g_{\xi}\circ\tilde{T}(x)$ is a spurious sample representing mixtures between modes in the observed data distribution $\eta$, $g_{\xi}$ being the decoder network. Hence to mitigate mode mixture, samples $x \in \Omega_{sing}(u)$ are rejected. Thus this is a rejection sampling scheme.
        \item[e.] Generate the sample image by $g_{\xi} \circ \tilde{T}(x)$ where $g_{\xi}$ is the decoder network.
    \end{enumerate}

    We now discuss the technicalities involved with the above steps.
    
    The semi-discrete OT map T induces a cell decomposition (a partition) of $\Omega$ of the form $\Omega=\bigcup_{i=1}^{n} W_{i}$. Thus corresponding to every $x \in \Omega$, there exists an $i \in \left\{1,2,\dots,n\right\}$ such that $x \in W_{i}$. Further, for every $i \in \left\{1,2,\dots,n\right\}$, every $x$ belonging to cell $W_{i}$ is mapped to the target $y_{i}$ by the optimal transport map T i.e.
    $$T(x) = y_{i} \textrm { if and only if } x \in W_{i}$$
    Consequently, we also have that $\mu\left(W_{i}\right)=\frac{1}{n}$.  
    
    
    Under quadratic loss, T is the gradient of the piecewise linear convex Brenier potential $$u_{h}: \Omega \rightarrow \mathbb{R}, u_{h}(x):=\max _{i=1}^{n}\left\{\pi_{h, i}(x)\right\}=\max _{i=1}^{n} \{{x^{T}}y_{i}+h_{i}\}$$ where $\pi_{h, i}(x)={x^{T}}y_{i}+h_{i}$ is the hyperplane corresponding to $y_{i} \in Y$. The projection of the graph of $u_{h}$ decomposes $\Omega$ into cells $W_{i}(h)$, each cell $W_{i}(h)$ is the projection of the supporting plane $\pi_{h, i}(x)$ i.e.
    $$
    W_{i}(h) = \left\{x \in \Omega \mid \nabla u_{h}(x)=y_{i}\right\}, i=1,2,\dots,n
    $$
    We often drop the reference to $h$ and refer to $W_{i}(h)$ as $W_{i}$, as in the previous paragraph. The height vector $h$ is the unique minimizer of the following convex energy
    $$E(h)=\idotsint_{S_{h}} \sum_{i=1}^{n} w_{i}(v) d v_{i}-\frac{1}{n}\sum_{i=1}^{n} h_{i}
    $$
    under the linear restriction that $\sum_{i=1}^{n} h_{i}=0$, where $S_{h}=\{v = (v_{1},v_{2},\dots,v_{n})\in \mathbb{R}^{n}\mid 0\leq v_{i} \leq h_{i},i=1,2,\dots,n\}$, $w_{i}(v)$ is the $\mu$-volume of $W_{i}(v)$ i.e. $\mu(W_{i}(v))=w_{i}(v)$ and $v=\left(v_{1},v_{2}.\dots,v_{n}\right)$ being the variable of integration. Following \citet{VarMinkowski}, $E(h)$ can be optimized by gradient descent method. The $\mu$-volume $w_{i}(h)$ of each cell $W_{i}(h),$ is estimated using conventional Monte Carlo method.

    
    To generate new samples, the semi-discrete OT map $T=\nabla u_{h}$ is extended to a piecewise linear (PL) mapping $\tilde{T}$ as follows. By representing the cells $W_{i}(h)$ by their $\mu$ -mass centers as $$c_{i}:=\int_{W_{i}(h)} x d \mu(x)$$ we obtain the point-wise map $f: c_{i} \mapsto y_{i}$.
    
    The Poincaré of the cell decomposition induces a triangulation of the centers $C=\left\{c_{i};i=1,2,\dots,n\right\}:$ if $W_{i} \cap W_{j} \neq \emptyset,$ then $c_{i}$ is connected with $c_{j}$ to form an edge $\left[c_{i}, c_{j}\right] .$ Similarly, if $W_{i_{0}} \cap W_{i_{1}} \cdots \cap W_{i_{k}} \neq \emptyset,$ then there is a $k$ -dimensional simplex $\left[c_{i_{0}}, c_{i_{1}}, \ldots, c_{i_{k}}\right]$. The simplicial complex formed by these simplices is a triangulation of $C$, denoted as $\mathcal{T}(C)$ . A triangulation $\mathcal{T}(\mathcal{Z})$ of $\mathcal{Z}$ is computed similarly.
    
    After drawing a random sample $x\sim\mu,$,with $\mu$ being the noise distribution, one can determine the simplex $\sigma$ in $\mathcal{T}(C)$ containing $x .$ Assuming the simplex $\sigma$ has $d+1$ vertices $\left\{c_{i_{0}}, c_{i_{1}}, \ldots, c_{i_{d}}\right\},$ the barycentric coordinates of $x$ in $\sigma$ is defined as $x=\sum_{k=0}^{d} \lambda_{k} c_{i_{k}},$ and $\sum_{k=0}^{d} \lambda_{k}=1$ with all $\lambda_{k}$ non-negative. Then the generated latent code of $x$ under this piecewise linear map is given by $$\tilde{T}(x)=\sum_{k=0}^{d} \lambda_{k} y_{i_{k}}$$
    No modes are lost and mode collapse is avoided since all of the $y_{i}$ s are used to construct the simplicial complex $\mathcal{T}(\mathcal{Z})$ in the support of the target distribution.

    During practical implementation, the $\mu$ -mass center $c_{i}$ is approximated by the mean value of all the Monte-Carlo samples inside $W_{i}(h)$ i.e. $$ \hat{c}_{i}=\frac{\sum_{x_{j} \in W_{i}} x_{j}} {\#\left\{x_{j} \in W_{i}\right\}}$$ where $x_{j} \sim \mu,j=1,2,\dots,N_{m}$ and $N_{m}$ is the number of Monte Carlo samples used in estimation. The connectivity information $\mathcal{T}(C)$ is too complicated to construct and to store in high dimensional space, thus $\mathcal{T}(C)$ is not explicitly built.
    
    In practice, the simplex $\sigma \in \mathcal{T}(C)$ containing $x$ is determined as follows: given a random point $x \in \Omega$, evaluate and sort its Euclidean distances to the centers $d\left(x, \hat{c}_{i}\right), i=1,2, \ldots, n$ in the ascending order. Suppose the first $d+1$ items are $\left\{d\left(x, \hat{c}_{i_{0}}\right), d\left(x, \hat{c}_{i_{1}}\right), \ldots, d\left(x, \hat{c}_{i_{d}}\right)\right\},$ then $\sigma$ is formed by $\left\{\hat{c}_{i_{k}}\right\} .$ The barycentric coordinates $\hat{\lambda}_{i_{k}}$ are estimated as $$\hat{\lambda}_{k}=\frac{d^{-1}\left(x, \hat{c}_{i_{k}}\right)} {\sum_{k=0}^{d} d^{-1}\left(x, \hat{c}_{i_{k}}\right)}$$
    This constitutes the backbone of the sample generation procedure of the AE-OT model.
    
    However, this may generate some spurious samples, when some of the $x$'s randomly generated from $\mu$ fall inside the singular set $\Omega_{sing}$, leading to mode mixture. To mitigate this problem, one needs to detect the singular set $\Omega_{sing}$ and remove the samples falling inside it.
    
    If there are multiple modes in the target distribution or the support of the target distribution of the optimal transport map T is concave, then, according to Figalli's theory, there will be singular sets $\Omega_{sing} \subset \Omega,$ where the Brenier potential $u_{h}$ is continuous but not differentiable, making its gradient map, i.e. the transport map $T=\nabla u_{h}$, discontinuous. In case of multimodality, which is the most common situation that occurs in practice, $\Omega \backslash \Omega_{sing}$ will consist of as many connected components as the number of modes, each of them mapped to a single mode. $\Omega_{sing}$ consists of codimension 1 facets of cells. If $W_{i}(h) \cap W_{j}(h) \subset \Omega_{sing},$ then the dihedral angle between two supporting planes $\pi_{h, i}$ and $\pi_{h, j}$ of $u_{h}$ is prominently large. Therefore, on the graph of Brenier potential, we pick the pairs of facets whose dihedral angles are larger than a given threshold, the projection of their intersection gives a co-dimension 1 cell in the singular set $\Omega_{sing}$. During the generation process, if a random sample $x$ is around $\Omega_{sing},$ it will be mapped by $\tilde{T}$ to the gaps among the modes. When generating new latent codes, we reject such samples, and this helps to prevent the mode mixture phenomenon.
    
    Specifically, given $x \sim \mu$, we can detect if it belongs to the singular set by checking the angles $\theta_{i_{k}}$ between $\pi_{i_{0}}$ and $\pi_{i_{k}}, k=1,2, \ldots, d$ as $$\theta_{i_{k}}=\left\langle y_{i_{0}}, y_{i_{k}}\right\rangle /\left\|y_{i_{0}}\right\| \cdot\left\|y_{i_{k}}\right\|$$ If all of the angles $\theta_{i_{k}}$ is larger than a threshold $\hat{\theta},$ we say $x$ belongs to the singular set and just reject it. Or else we select a subset $\left\{\pi_{i_{k}}\right\}$ with $\theta_{i_{k}} \leq \hat{\theta},$ denoted as $\left\{\pi_{\hat{i}_{k}}, k=0,1, \ldots, d_{1}\right\} .$ Then we can compute $\lambda_{k}=d^{-1}\left(x, \hat{c}_{i_{k}}\right) / \sum_{j=0}^{d_{1}} d^{-1}\left(x, \hat{c}_{i_{j}}\right)$ and $\tilde{T}(x)=\sum_{k=0}^{d_{1}} \lambda_{k} T\left(\hat{c}_{\hat{i}_{k}}\right) .$ Intuitively, $\tilde{T}(\cdot)$
    smooths the discrete function $T(\cdot)$ in regions where latent codes are dense and keeps the discontinuity of $T(\cdot)$ where latent codes are very sparse. In this manner AE-OT avoids generating spurious latent code and thus improves the quality of generated samples.
    
    We would like to focus on a few problems of the AE-OT sample generation method:
    
    \begin{itemize}
        \item Detection of the singular set $\Lambda$ is based on the observation that sharp ridges between adjoining hyperplanes of $\hat{u_{h}}$ is indicated by large dihedral angles, and a thresholding parameter $\theta$ is used to determine which values of the angles should be considered prominently large, acting as a tuning parameter to be determined separately for different datasets. This is however just a proxy for finding the points of non-differentiability of $u_{h}$ in absence of any direct or better method. There is no principled method for finding an optimal value of $\theta$ to be used except for trying out different values and seeing which one gives best results, which is computationally expensive. There is no data-dependent intuition regarding what should be a good choice of $\theta$.
        \item A very computationally expensive rejection sampling scheme has been proposed to ensure mode mixture does not occur in the generated samples, which must be implemented for every choice of the threshold $\theta$ that we want to try out. It is very difficult to get large number of generated samples using this rejection sampling scheme, since a large of samples is rejected in practice, even when the choice of $\theta$ is close to optimal.
        \item Computing and storing the entire connectivity information corresponding to the simplicial complex $\mathcal{T}(C)$ is infeasible, even for moderately large dimensions of $\mathcal{Z}$. The algorithm approximates the true simplicial complex $\mathcal{T}(C)$ by constructing a simplicial complex having simplices of maximum degree $d$ (a $d$ dimensional simplex is defined using $d+1$ points). Although not explicitly stated in \citet{AEOT}, choice of $d$ is important to ensure the approximation is sufficiently accurate; too small a value of $d$ will lead to loss in accuracy of approximation while too large  a value of $d$ leads to an approximation which cannot be computed and stored in practice due to computational limitations.
        \item An additional source of error that creates a difference between theory and practice is that the barycentric coordinates $\lambda_{k}, k=0,1,\dots,d$ need to be estimated since exact computation is again infeasible due to computational limitations.
    \end{itemize}
\end{enumerate}

\subsection{Our modification: Convex Smoothed AE-OT model}

We were inspired to develop a generative model which borrows largely from the AE-OT model, but makes improvements to the sample generation method of AE-OT based on the idea of smoothing the convex Brenier potential function $u_{h}$. The idea for smoothing the convex function $u_{h}$ is based on the idea of smoothing non-smooth convex estimators proposed in \citet{Convex}. The smoothed function is convex, Lipschitz continuous and differentiable everywhere, gives rise to an optimal transport map $\hat{T}$ that approximates T and is continuous everywhere. Further, T can be represented by a deep neural network with sufficient expressibility to arbitrary accuracy. We can control the degree of approximation based on an uniform error bound $\epsilon $ that is user-specified and is an interpretable parameter, unlike the tuning parameter $\theta$ discussed here. 

The use of $\tilde{T}(\cdot)$ is to primarily smooth the discrete function $T(\cdot)$ and allow us to generate new samples. We are motivated by this idea of smoothing, but we smooth $\hat{u_{h}}$ to remove non-differentiability of the function at certain points. On obtaining the gradient of this smoothed Brenier potential function, we automatically obtain a function capable of transforming any random sample $x$ from the noise distribution $\mu$ into a latent vector $z$ in $\mathcal{Z}$.

We propose to approximate the piecewise affine function $u_{h}(x)=\underset{i=1,2,\dots,n}{\max}x^{T}y_{i}+h_{i}$ or more precisely $u_{\hat{h}(x)}$ ($\hat{h}$ is a the estimate of $h$ obtained using Algorithm 1 of \citet{AEOT}; here we consider that we either know the true $h$ or are able to estimate $h$ using $\hat{h}$ very accurately, so we will refer to $u_{h}$ only in our discussion) by a convex smooth differentiable function to a sufficient degree of accuracy, say $\hat{u_{h}}(x)$. This accuracy is defined by a uniform bound $\epsilon$ on the difference of the true and approximated functions i.e.
 $\sup _{\mathrm{x}}\left|u_{h}(x)-\hat{u_{h}}(x)\right| \leq \epsilon$). $\hat{u_{h}}(x)$ can play the role of the Brenier potential function so that the gradient of this approximated function will be the optimal transport map between the noise distribution and an appropriate approximation of the discrete empirical distribution of the embedded latent vectors in the latent space.
 
 Here a question may arise as to whether the gradient $\hat{T}$ of this smooth convex approximated function $\hat{u_h}(x)$ is indeed an optimal transport map, since all functions do not qualify to be optimal transport maps. In this respect we refer to a result in Brenier (\citet{brenier1987polar}),  originally proved by Ryff (\citet{ryff1965orbits}), which basically states that any convex function is an optimal transport map between two distributions under quadratic loss. In Section \ref{sec:Theoretical validation of the Convex Smoothed AE-OT model} of this paper, we also investigate how close this OT map $\hat{T}$ is close to the true OT map $T$, even though they are fundamentally different due to the former being a continuous function while the latter being a discontinuous one.

If this can be done, then the semi-discrete optimal transport problem between a continuous noise distribution and a discrete empirical distribution on observed latent codes is now transformed to an optimal transport problem between two continuous distributions, the source distribution being the noise distribution as before but the target distribution is a continuous approximation (hopefully good) of the discrete distribution.

\subsubsection{Justification of the modification}

Following the development in section 3.2 of the paper \cite{Convex} based on the convex optimization theory of \cite{nesterov1998introductory}, we have that

$$u_{h}(x_{j})=\underset{i=1,2,\dots,n}{\max}x_{j}^{T}y_{i}+h_{i}=\underset{\Delta_{n}}{\sup}\sum_{i=1}^{n}w_{i}\left(x_{j}^{T}y_{i}+h_{i}\right)= \underset{\Delta_{n}}{\sup} \langle Az_{j}^{T},\textbf{w} \rangle$$

where $$\Delta_{n}=\lbrace{ \textbf{w}=(w_{1},w_{2},\dots,w_{n}) \in \mathbb{R}^{n}:\sum_{i=1}^{n}w_{i}=1 , w_{i}\geq 0, i=1,2,\dots,n\rbrace}$$ and
$z_{j}$ is the $j$-th row of Z i.e. $z_{j} = (1,x_{j})^{T}$.

We require the notion of a proximity function. A proximity function (or prox function) $\rho(.)$ defined on $\Delta_{n}$ is a continuous strongly convex function with strong convexity parameter $m=1$  i.e. $$\rho(y) \geq \rho(x)+\nabla \rho(x)^{T}(y-x)+\frac{1}{2}\|y-x\|_{2}^{2}$$ for any $x,y \in \Delta_{n}$.

Let us define $\hat{u_{h}}(x;\tau)=\underset{\Delta_{n}}{\sup} \langle Az^{T},\textbf{w}\rangle - \tau\rho\textbf(w)$ where $z=(1,x)^{T}$. Often we will drop reference to $\tau$ when it is understood. The following results are the basis of the proposed method:

\begin{lemma}
For any fixed $\tau>0,$ the function $\hat{u_{h}}(x;\tau)$ is convex and is continuously differentiable in z. Its gradient is given by $\nabla \hat{u_{h}}(x;\tau) =A^{T} \hat{\mathbf{w}}^{\tau},$ where $$\hat{\mathbf{w}}^{\tau} \in \argmax _{\mathbf{w} \in \Delta_{n}}\{\langle Az^{T}, \mathbf{w}\rangle-\tau \rho(\mathbf{w})\} $$ Furthermore, the gradient map $\mathrm{z} \mapsto \nabla \hat{u_{h}}(x;\tau)$ is Lipschitz continuous with parameter $\frac{\|A\|^{2}}{\tau}$.
\end{lemma}

\begin{lemma}
For any $\tau \geq 0,$ the perturbation $\hat{u_{h}}(x;\tau)$ of $\hat{u_{h}}(x;0)=u_{h}(x)$ satisfies the following
uniform bound over z:
$$
u_{h}(x)-\tau \hspace{2pt}\sup _{\mathbf{w} \in \Delta_{n}} \rho(\mathbf{w}) \leq \hat{u_{h}}(x;\tau) \leq \hat{u_{h}}(x;0) = u_{h}(x)
$$
\end{lemma}

We initially test our idea using a particular choice of the proximity function, namely the entropy prox function. The entropy prox function on the unit simplex $\Delta_{n}$ is given by $\rho(\mathbf{w})=$ $\sum_{i=1}^{n} w_{i} \log \left(w_{i}\right)+\log n$.

For the entropy prox function, we are able to obtain a closed form solution for $\hat{u_{h}}(x;\tau)$. We have that
\begin{align}
&\underset{\textbf{w}\in \Delta_{n}}{\argmax} \sum_{i=1}^{n}w_{i}\left(x_{j}^{T}y_{i}+h_{i}\right)-\tau \left(\sum_{i=1}^{n} w_{i} \log \left(w_{i}\right)+\log n\right)\\
&= \left( \frac{\exp c_{1}}{\sum_{i=1}^{n}\exp c_{i}},\frac{\exp c_{2}}{\sum_{i=1}^{n}\exp c_{i}},\dots,\frac{\exp c_{n}}{\sum_{i=1}^{n}\exp c_{i}} \right)
\end{align}
where $c_{i} = \frac{\mathbf{y}_{i}^{T} \mathbf{x}+h_{i}}{\tau}$. Hence we have
$$\hat{u_{h}}(x;\tau)=\tau \log \left(\sum_{i=1}^{n} \exp \left(\frac{\mathbf{y}_{i}^{T} \mathbf{x}+h_{i}}{\tau}\right)\right)-\tau \log n$$ and

$$\nabla \hat{u_{h}}(x;\tau) = \frac{\sum_{i=1}^{n}y_{i}\exp \hspace{2pt}c_{i}}{\sum_{i=1}^{n}\exp \hspace{2pt}c_{i}}$$

Choosing $\tau=\frac{\epsilon}{\log n}$, we get the optimal transport map as $$\hat{T}(x)=\nabla \hat{u_{h}}(x;\tau) = \frac{\sum_{i=1}^{n}y_{i}\exp \hspace{2pt}c_{i}}{\sum_{i=1}^{n}\exp \hspace{2pt}c_{i}}$$

Thus given any noise sample $x$, $\hat{T}(x)$ is the generated latent vector.

\section{Algorithm }\label{sec:Algorithm}

The algorithm to compute the Optimal transport map for our modified AE-OT model is exactly the same as Algorithm \ref{Alg.1} as proposed in \citet{AEOT}. However the algorithm for latent code generation by smoothing the semi-discrete OT map is different and will replace Algorithm 2 (\citet{AEOT}) of the AE-OT methodology (dealing with piecewise linear extension of the Semi-Discrete Optimal Transport Map) for generating new latent codes. We provide the algorithm here as Algorithm \ref{Alg.1} for sake of completeness.

\begin{algorithm}
\caption{Semi-Discrete OT Map}
\label{Alg.1}
\begin{algorithmic}[1]
\Require Latent codes $Y=\left\{y_{i}\right\}_{i \in \mathcal{I}},$ empirical latent code distribution $\nu=\frac{1}{|\mathcal{I}|} \sum_{i \in \mathcal{I}} \delta_{y_{i}},$ number of Monte Carlo samples $N,$ positive integer $s$
\Ensure Optimal transport map $T$ ( ).
\State Initialize $h=\left(h_{1}, h_{2}, \ldots, h_{|\mathcal{I}|}\right) \leftarrow(0,0, \ldots, 0)$
\Repeat 
\State Generate $N$ uniformly distributed samples $\left\{x_{j}\right\}_{j=1}^{N}$
\State Calculate $\nabla E=\left(\hat{w}_{i}(h)-\nu_{i}\right)^{T}$
\State Update $h$ by Adam algorithm with $\beta_{1}=0.9, \beta_{2}=0.5$
\State $h= h-\operatorname{mean}(h)$
\If{$E(h)$ has not decreased for $s$ steps}
\State $N \gets N \times 2$
\EndIf
\Until{Converge}
\State OT $\operatorname{map} T(\cdot) \leftarrow \nabla\left(\max _{i}\left(\cdot, y_{i}\right\rangle+h_{i}\right)$
\end{algorithmic}
\end{algorithm}

Let $P$ be a matrix of dimension $n$ $\times$ $d$ (where $n$ is the number of observed latent vectors embedded in the latent space, and $d$ is the dimension of the latent space) having the embedded latent vector $y_{i}$ as its $i$-th row. Let $h=(h_{1},h_{2},\dots,h_{n})$ be the same as in the AE-OT methodology. 

We find out the optimal value of $h$ first using Algorithm 1 under the AE-OT methodology as before. One point we would like to mention is that a natural stopping criterion for Algorithm 1 to terminate is when the energy function $E$ either does not change for a few steps, or the successive reductions in its value is very small. However that involves the calculation of $E$ at each step. To avoid the additional computational burden, we can alternatively use the norm of the gradient to specify a stopping criterion. When the value of $E$ is near a local minimum, the gradient should be very small and consequently the norm of the gradient should be close to zero. So we terminate Algorithm 1 when the norm of the gradient is sufficiently small (say less than 0.002).

Having obtained the optimal $h$, we define the matrix A as $A = \left[h,P\right]$ i.e. stacking $h$ and $P$ horizontally. Let us say we want to generate $N$ samples. Then we draw $N$ i.i.d. samples $x_{1},x_{2},\dots,x_{N}$ from the noise distribution and define Q to be a matrix whose $i$-th row is $x_{i}$. We append a column of ones at the left of this matrix Q to obtain $Z$ which is a $N$ $\times$ $d+1$ matrix. Next we obtain $I=AZ^{T}$ which is a $n$ $\times$ $N$ dimensional matrix with $(i,j)$-th element equal to $x_{j}^{T}y_{i}+h_{i}$. Then we obtain $I_{scaled}=\frac{1}{\tau}I$ and apply the softmax function over each column of $I_{scaled}$ to obtain $W$. The softmax function is defined as $$\sigma(\boldsymbol{t})=\left(\frac{\exp \hspace{2pt}t_{1}}{\sum_{i=1}^{n}\exp \hspace{2pt}t_{i}},\frac{\exp \hspace{2pt}t_{2}}{\sum_{i=1}^{n}\exp \hspace{2pt}t_{i}},\dots,\frac{\exp \hspace{2pt}t_{n}}{\sum_{i=1}^{n}\exp \hspace{2pt}t_{i}}\right)$$
Here, W is the matrix of optimized weights with the weights corresponding to the noise sample $x_{j}$ in the $j$-th column of W. Then we obtain $G=A^{T}W$ which gives the matrix of gradients with respect to each column of Z. Removing the first row of G we obtain the matrix of generated samples $X_{gen}$, with the $j$-th column being the generated sample corresponding to $x_{j}$.

The proposed modified algorithm is summarised below in Algorithm \ref{Alg.New}.

\begin{algorithm}
\caption{Generate Latent Code}
\label{Alg.New}
\begin{algorithmic}[1]
\Require 1. Optimal value of $h=(h_{1},h_{2},\dots,h_{n})$ from Algorithm 1 of the AE-OT algorithm\newline
2. Number of samples to generate $N$\newline
3. Noise distribution to sample from: $\nu$\newline
4. Matrix P of dimension $n$ $\times$ $d$ (where $n$ is the number of observed latent vectors embedded in the latent space, and $d$ is the dimension of the latent space) having the embedded latent vector $y_{i}$ as its $i$-th row\newline
5. Uniform error bound on approximating true Brenier potential $\epsilon$
\Ensure Generated latent code $X_{gen}$.
\State Define the matrix A as $A = \left[h,P\right]$ i.e. stacking $h$ and $P$ horizontally.
\For{i in 1,2,\dots,N}
\State Sample $x_{i} \sim \nu$
\EndFor
\State Define Q to be a matrix whose $i$-th row is $x_{i}$ $i=1,2,\dots,N$
\State Append a column of ones at the left of this matrix Q to obtain $Z$, which will be a $N$ $\times$ $d+1$ matrix.
\State Compute $I=AZ^{T}$ which is a $n$ $\times$ $N$ dimensional matrix with $(i,j)$-th element equal to $x_{j}^{T}y_{i}+h_{i}$.
\State Define $\tau=\frac{\epsilon}{\log n}$
\State Compute $I_{scaled}=\frac{1}{\tau}I$
\State Apply the softmax function over each column of $I_{scaled}$ to obtain $W$
\State Compute $G=A^{T}W$.
\State Remove the first row of G to obtain the $d \times N$matrix of generated samples $X_{gen}$, with the $j$-th column being the generated sample corresponding to $x_{j}$, $j=1,2,\dots,N$.
\end{algorithmic}
\end{algorithm}

\subsection{Optimal choice of $\epsilon$}

Algorithm \ref{Alg.New}. requires a user specified hyperparameter $\epsilon$, which represents the uniform error bound on the approximation of the true Brenier potential function $u_{h}(x)$ by the estimate $\widehat{u_{h}(x)}$, since
$$
\sup _{\mathrm{x}}\left|u_{h}(x)-\hat{u_{h}}(x)\right| \leq \epsilon
$$
One might be tempted to choose $\epsilon$ as small as possible, in order to ensure that the error in approximation is minimized. However, such an approach, in the limit when $\epsilon$ tends to 0, will lead to a scenario where, irrespective of the sample $x$ generated from the noise distribution $\mu$, the latent vector $\hat{T}(x)$ will be exactly equal to one of the observed latent vectors $y_{1},y_{2},\dots,y_{n}$. Although this leads to the the generated latents and hence the generated images to have exactly the same distribution as the observed images, it defeats our purpose of generating \say{new} samples. On the other hand, a large value of $\epsilon$ would lead to generation of samples very dissimilar from the observed data. 

To mitigate this problem and provide a reasonable choice for $\epsilon$ which provides a trade-off between the two extreme scenarios, one may use the following strategy:

Choose a sequence of $\epsilon$ values, varying from extremely large to extremely small. For each choice of $\epsilon$, generate $n$ samples $t_{1},t_{2},\dots,t_{n}$ based on Algorithms \ref{Alg.1} and \ref{Alg.New}. One then has two discrete (multivariate) distributions in hand: the distribution $\phi=\frac{1}{n}\sum_{l=1}^{n}\delta_{t_{l}}$ of the generated samples and the distribution $\eta=\frac{1}{n}\sum_{l=1}^{n}\delta_{o_{l}}$ of the observed samples.

We assume that the generated and observed samples are drawn from underlying distributions $\mathcal{P}$ and $\mathcal{Q}$, respectively. A statistical test of similarity of these two distributions $\mathcal{P}$ and $\mathcal{Q}$ based on $\phi$ and $\eta$ would provide a measure of similarity between the two distributions, by means of the computed p-value. A very large p-value indicates a large degree of similarity between the two distributions and we expect to obtain such large p-values corresponding to extremely small values of $\epsilon$. On the other hand, a very small p-value will indicate a large degree of dissimilarity  between the two distributions, and we expect to obtain such small p-values corresponding to extremely large values of $\epsilon$. We fix a threshold $\alpha$ for the p-value (equivalent to fixing the significance level of the test) to reasonably indicate the point of transition from dissimilarity to similarity of the two distributions based on the sequence of $\epsilon$ values. A reasonable choice of $\epsilon$ would be one which leads to a p-value approximately equal to $\alpha$. 

A good and popular choice of a statistical test of equality of multivariate distributions is the Maximum Mean Discrepancy (MMD) Test (\citet{MMDtest}), which has been used for comparing the the generated sample distribution to a reference distribution such as the observed sample distribution in order to assess the performance of generative models like GANs (for e.g. in \citet{ModelMMDtest}). At a high level, the test is based on maximizing the difference between the expectation of a suitable function evaluated on the two datasets separately, and rejecting the null hypothesis of equality if the difference is significantly large. A brief introduction to the MMD test is given following this subsection.

\subsubsection{The Maximum Mean Discrepancy Test}

Let $k$ be the kernel of a reproducing kernel Hilbert space (RKHS) $\mathcal{H}_{k}$ of functions on the space $\mathcal{X}$ of observed and generated samples. $k$ is assumed to be measurable and bounded, $\sup _{x \in \mathcal{X}} k(x, x)<\infty$. The Maximum Mean Discrepancy (MMD) in $\mathcal{H}_{k}$ between the two distributions $\mathcal{P}$ and $\mathcal{Q}$ over $\mathcal{X}$ is defined in the following manner (\citet{MMDtest}):
$$
\operatorname{MMD}_{k}^{2}(\mathcal{P}, \mathcal{Q}):=\mathbb{E}_{t, t^{\prime}}\left[k\left(t, t^{\prime}\right)\right]+\mathbb{E}_{i, i^{\prime}}\left[k\left(i, i^{\prime}\right)\right]-2 \mathbb{E}_{t, i}[k(t, i)]
$$
where $t, t^{\prime} \stackrel{\text { iid }}{\sim} \mathcal{P}$ and $i, i^{\prime} \stackrel{\text { iid }}{\sim} \mathcal{Q} .$ 

Given the empirical distributions $\phi$ and $\eta$ corresponding to the $t_{l}$'s and the $o_{l}$'s, respectively, an unbiased estimator of $\operatorname{MMD}(\mathcal{P}, \mathcal{Q})$ with nearly minimal variance among unbiased estimators is
$$
\widehat{\mathrm{MMD}}_{\mathrm{U}}^{2}(\phi, \eta):=\frac{1}{\left(\begin{array}{c}
n \\
2
\end{array}\right)} \sum_{l \neq l^{\prime}} k\left(t_{l}, t_{l^{\prime}}\right)+\frac{1}{\left(\begin{array}{c}
n \\
2
\end{array}\right)} \sum_{m \neq m^{\prime}} k\left(o_{m}, o_{m^{\prime}}\right)-\frac{2}{\left(\begin{array}{c}
n \\
2
\end{array}\right)} \sum_{l \neq m} k\left(t_{l}, o_{m}\right)
$$

Following \citet{MMDtest}, we conduct a hypothesis test with null hypothesis $H_{0}: \mathcal{P}=\mathcal{Q}$ and alternative $H_{1}: \mathcal{P} \neq \mathcal{Q},$ using test statistic $n \widehat{\mathrm{MMD}}_{\mathrm{U}}^{2}(\phi, \eta) .$ For the chosen significance level $\alpha,$ we choose a test threshold $c_{\alpha}$ and reject $H_{0}$ if $n \widehat{\mathrm{MMD}}_{\mathrm{U}}^{2}(\phi, \eta)>c_{\alpha}$

Under $H_{0}: \mathcal{P}=\mathcal{Q}, n \widehat{\mathrm{MMD}}_{\mathrm{U}}^{2}(\phi, \eta)$ converges asymptotically to a distribution that depends on the unknown distribution $\mathcal{P}$; we thus cannot evaluate the test threshold $c_{\alpha}$ in closed form. We instead estimate a data-dependent threshold $\hat{c}_{\alpha}$ via permutation, thus using a bootstrap/ permutation test. This gives us a distribution-free test.

Let $T$ and $I$ represent the collection of generated and observed samples respectively. The permutation test involves randomly partitioning the data $T \cup I$ into $T^{\prime}$ and $I^{\prime}$ many times (with $\phi^\prime$ and $\eta^\prime$ representing the corresponding empirical distributions), evaluating $n \widehat{M M D}_{U}^{2}\left(\phi^{\prime}, \eta^{\prime}\right)$ on each split, and estimating the $(1-\alpha)$-th quantile $\hat{c}_{\alpha}$ from these samples. 

In practice we use the implementation of the MMD test available in the Python package \texttt{alibi-detect} (\citet{alibi}) (Documentation available at \url{https://docs.seldon.io/projects/alibi-detect/en/stable/methods/mmddrift.html})

\paragraph{Choice of kernel function:} The test requires the choice of a kernel function for comparing the similarity of samples from the generated collection and the observed collection. Many kernels, including the popular Gaussian Radial Basis Function (RBF) kernel, are characteristic, which implies that the MMD is a metric, and in particular that $\operatorname{MMD}_{k}(\mathcal{P},\mathcal{Q})=0$ if and only if $\mathcal{P}=\mathcal{Q},$ so that tests with any characteristic kernel are consistent. The RBF kernel is given by,
$$
k\left(t, i\right)=\exp \left(-\frac{\left\|t- o\right\|^{2}}{2 \sigma^{2}}\right)
$$
However different characteristic kernels will yield different test powers for finite sample sizes. In this paper, we stick to using the RBF kernel, with the kernel bandwidth $\sigma$ chosen as the median of the $L_{2}$ norms of the pairwise differences between the $t_{l}$'s and the $o_{l}$'s i.e. $\sigma = \textrm{median }\left\{\|t_{l}-o_{m}\|_{2};l,m=1,2,\dots,n\right\}$ 

\paragraph{Obtaining the optimal $\mathbf{\epsilon}$ :} Let the initial sequence of length $s$ of possible $\epsilon$ values be $\epsilon_{1,1},\epsilon_{1,2},\dots,\epsilon_{1,s}$, arranged in increasing order. Corresponding to $\epsilon_{1,l}$ , $l \in {1,2,\dots,s}$, we obtain the collection of generated samples $T_{1,l}$ and the corresponding empirical distribution of generated samples $\phi_{1,l}$. If $\eta$ is the empirical distribution of the observed samples, we perform the permutation test based on the MMD test statistic computed using $\phi_{1,l}$ and $\eta$, and obtain the corresponding p-value $\textrm{\textbf{pval}}_{1,l}$, as described earlier. If any one of the p-values is approximately equal to the chosen significance level $\alpha$, then the corresponding $\epsilon$ value is the optimal choice. Otherwise, assume that there exists $l_{1}^{*},l_{2}^{*} \in {1,2,\dots,s}$ such that $\textrm{\textbf{pval}}_{1,l_{1}^{*}}>\alpha$ and $\textrm{\textbf{pval}}_{1,l_{2}^{*}}<\alpha$. Then the optimal choice of $\epsilon$ i.e. $\epsilon_{opt}$ belongs to the interval $\left(\epsilon_{1,l_{1}},\epsilon_{1,l_{2}}\right)$. Then consider the sequence of $\epsilon$ values $\epsilon_{2,m}=\epsilon_{1,l_{1}^{*}}+m\times\frac{\epsilon_{1,l_{2}^{*}}-\epsilon_{1,l_{1}^{*}}}{10},m=1,2,\dots,10$. Again, we obtain the generated samples $T_{2,m}$ corresponding to the $\epsilon_{2,m}$'s, perform the MMD tests and calculate the p-values $\textrm{\textbf{pval}}_{2,m}$'s. If any one of the p-values is approximately equal to the chosen significance level $\alpha$, then the corresponding $\epsilon$ value is the optimal choice $\epsilon_{opt}$. Otherwise, we continue to proceed in a similar manner till such an $\epsilon$ is obtained. Whether any of the p-values are approximately equal to the desired $\alpha$ value can be checked by specifying a tolerance threshold $\delta$ such that if the p-value \textbf{pval} is such that $|\textrm{pval}_{l}-\alpha|<=\delta$, then the corresponding choice of $\epsilon$ is declared to be the optimal choice. If there are multiple such choices, then any one of them might be used, as it makes little difference in practice.

The proposed procedure is summarised below in Algorithm \ref{Alg. eps}.

\begin{algorithm}[H]
\caption{Obtain Optimal choice of $\epsilon$ and Generated Latent Codes}
\label{Alg. eps}
\begin{algorithmic}[1]
\Require 1. Optimal value of $h=(h_{1},h_{2},\dots,h_{n})$ from Algorithm 1 of the AE-OT algorithm\newline
2. Number of samples to generate $n$\newline 
3. Noise distribution to sample from: $\nu$\newline
4. Matrix P of dimension $n$ $\times$ $d$ (where $n$ is the number of observed latent vectors embedded in the latent space, and $d$ is the dimension of the latent space) having the embedded latent vector $y_{i}$ as its $i$-th row\newline
5. List of uniform error bounds on approximating true Brenier potential E=$\left\{\epsilon_{l},l=1,2,\dots,s\right\}$\newline
6. Chosen significance level $\alpha$\newline
7. Tolerance threshold for p-value $\delta$
\Ensure Optimal choice of the uniform error bound $\epsilon_{opt}$ and the corresponding Generated latent code $X_{gen}$.
\State Set $\epsilon_{opt}=0$, $\textrm{\textbf{pval}} = 1$, $\textrm{\textbf{pval}}_{lower} = 1$, $\textrm{\textbf{pval}}_{upper} = 1$
\While{$|\textrm{\textbf{pval}}-\alpha|>\delta$}
\State Set $s=\textrm{card}(E)$
\For{l in 1,2,\dots,s}
\State Run Algorithm \ref{Alg.New} with parameters $h,n,\nu$, P and $\epsilon_{l}$ to generate latent codes $T_{l}$.
\State Perform the MMD test based on $T_{l}$ and $I$. Store the p-value obtained $\textrm{pval}_{l}$.
\If{$|\textrm{pval}_{l}-\alpha|<=\delta$}
    \State Set $\textrm{\textbf{pval}} = \textrm{pval}_{l}$
    \State $\epsilon_{opt}=\epsilon_l$
    \State \textbf{break}
\ElsIf{$\textrm{pval}_{l}<\alpha$}
    \State Set $\epsilon_{lower} = \epsilon_l$
\Else{$\textrm{ pval}_{l}>\alpha$}
    \State Set $\epsilon_{upper} = \epsilon_l$
\EndIf
\EndFor
\If{$|\textrm{pval}_{l}-\alpha|>\delta$}
    \State Set E = $\left\{\epsilon_{m}=\epsilon_{lower}+m\times\frac{\epsilon_{upper}-\epsilon_{lower}}{10},m=1,2,\dots,10\right\}$
\EndIf
\EndWhile
\State Run Algorithm \ref{Alg.New} with parameters $h,n,\nu$, P and $\epsilon_{opt}$ to generate latent codes $X_{gen}$
\end{algorithmic}
\end{algorithm}

\section{Theoretical validation of the Convex Smoothed AE-OT model}\label{sec:Theoretical validation of the Convex Smoothed AE-OT model}

Following the discussion in Section \ref{sec:AE-OT to modified AE-OT Framework}, the true Brenier potential map $u(.)$ corresponding to the true optimal transport map $T=\nabla u$ between $\mu$, the noise distribution, and $\nu$, the empirical latent code distribution is parametrized by a \say{height} vector $h=(h_{1},h_{2},\ldots,h_{n})$ and is of the form,
\begin{equation*}\label{true u}
u_{h}(\mathbf{x})=\underset{i=1,2,\dots,n}{\max}\left\{\mathbf{x}^{T}\mathbf{y}_{i}+h_{i}\right\}
\end{equation*}
For a given dataset, once the autoencoder has been trained, $y_{1},y_{2},\dots,y_{n}$ are known constants and the height vector $h=(h_{1},h_{2},\ldots,h_{n})$ is an unknown parameter with linear restriction $\sum_{i=1}^{n}h_{i}=0$. 

Based on the entropy prox(imity) function and the theory of smoothing a convex non-smooth function as presented in \citet{nesterov1998introductory}, the smooth approximation of $u_{h}(x)$ is given by 

\begin{equation*}\label{estimated u}
\begin{aligned}
\widehat{u_{h}}(\mathbf{x};\tau)&=\tau \log \left(\sum_{i=1}^{n} \exp \left(\frac{\mathbf{x}^{T} \mathbf{y}_{i}+h_{i}}{\tau}\right)\right)-\tau \log  n \\
&= \tau \hspace{2pt}\log \left(\sum_{i=1}^{n} \exp c_{i}\right)-\tau \hspace{2pt} \log  n
\end{aligned}
\end{equation*}

where $c_{i} = \frac{\mathbf{x}^{T} \mathbf{y}_{i}+h_{i}}{\tau}$ and $\tau$ is a quantity controlling the degree of accuracy of the approximation. Following Equation (21) of \citet{Convex}, we have an uniform error bound as follows:
\begin{equation*}\label{error bound}
    \sup _{\mathbf{x}}\left|u_{h}(\mathbf{x})-\widehat{u_{h}}(\mathbf{x};\tau)\right| \leq \tau \hspace{2pt} \log n
\end{equation*}

If the user specified upper bound on $\sup _{\mathbf{x}}\left|u_{h}(\mathbf{x})-\widehat{u_{h}}(\mathbf{x};\tau)\right|$ is $\epsilon$, then $\tau$ is chosen to be any positive real number less than or equal to $\frac{\epsilon}{\log n}$. For definiteness, we set $\tau = \frac{\epsilon}{\log n}$.


Under the assumption that the true height vector $h$ is recovered by Algorithm \ref{Alg.1}, the aforementioned result thus provides a bound on the error of approximation of the true Brenier potential map $u_{h}(x)$ by the smooth approximate Brenier potential map $\widehat{u_{h}}(x;\tau)$ constructed in the convex smoothed AE-OT model, with the bound being a decreasing function of $\tau$, and hence a decreasing function of $\epsilon$. Our objective is to obtain a similar result regarding the accuracy of approximation of the true OT map $\nabla u_{h}(x)$ by the approximate OT map $\nabla \widehat{u_{h}}(x;\tau)$.

In this section, we prove a bound on the error of approximation of the true OT map $\nabla u_{h}(x)$ by the approximate OT map $\nabla \widehat{u_{h}}(x;\tau)$ constructed in the convex smoothed AE-OT model. More specifically, we prove the following result:

\begin{theorem}
Let the $d$-dimensional noise distribution $\mu$ of the Convex Smoothed AE-OT model have convex and bounded support $\mathcal{X}$. Further, assume that Algorithm \ref{Alg.1} is able to recover the true value of the parameter $h$. Then, for any $\mathbf{x} \in \mathcal{X}$,
\begin{equation*}\label{subsec:final result}
\|\nabla u_{h}(\mathbf{x})-\nabla \widehat{u_{h}}(\mathbf{x})\|_{\mathrm{L}^{2}} \leq K \times \left(\log {n}\right)^{1 / 2} \times \tau^{1 / 2}
\end{equation*}
where $K$ is a constant which depends only on $\mathcal{X}$ and is independent of $\tau$.
\end{theorem}

Theorem \ref{subsec:final result} thus proves that the $L^{2}$ norm of the difference between the OT map $\nabla u_{h}(x)$ and the smoothed OT map $\nabla \widehat{u_{h}}(x;\tau)$ is bounded by a decreasing function of the error bound $\tau$, and hence by a decreasing function of $\epsilon$, as shown later by substituting $\tau=\frac{\epsilon}{\log n}$.

In order to prove Theorem \ref{subsec:final result}, we require the following result (Proposition 3.7) from \citet{Gradientdiffnorm}:
\begin{proposition}\label{sec:Proposition for final result}
 Let $f$ and $g$ be convex functions on a bounded convex set $\mathcal{X},$ then
\begin{equation*}\label{result}
\|\nabla f-\nabla g\|_{\mathrm{L}^{2}} \leq 2 C_{\mathcal{X}}\|f-g\|_{\infty}^{1 / 2}\left(\|\nabla f\|_{\infty}^{1 / 2}+\|\nabla g\|_{\infty}^{1 / 2}\right)
\end{equation*}
where $C_{\mathcal{X}}$ depends only on $\mathcal{X}$.
\end{proposition}

We now proceed to prove Theorem \ref{subsec:final result}.

\begin{proof}[Proof of Theorem \ref{subsec:final result}]
We observe that both $u_{h}(\mathbf{x})$ and $\widehat{u_{h}}(\mathbf{x};\tau)$ ($\tau>0$) are convex functions. Convexity of $u_{h}(\mathbf{x})$ follows from the fact that $u_{h}(\mathbf{x})$ is a piecewise linear function and the convexity of $\widehat{u_{h}}(\mathbf{x};\tau)$ ($\tau>0$) follows from Section 3 of \citet{Convex}. Further, the domain $\mathcal{X}$ of both $u_{h}(\mathbf{x})$ and $\widehat{u_{h}}(\mathbf{x};\tau)$ ($\tau>0$) is assumed to be convex and bounded. Hence, if we choose $f=u_{h}(\mathbf{x})$ and $g=\widehat{u_{h}}(\mathbf{x};\tau)$, the conditions for applying Proposition \ref{sec:Proposition for final result} are satisfied.

Now, the gradient of $u_{h}(\mathbf{x})$ is given by $$\nabla u_{h}(\mathbf{x}) = y_{m}$$ where $m \in \left\{1,2,\ldots,n\right\}$ is such that $\mathbf{x}^{T}\mathbf{y}_{i}+h_{i}$ is maximized over all $i \in \left\{1,2,\dots,n\right\}$ when $i=m$.

Again, the gradient of $\widehat{u_{h}}(\mathbf{x};\tau)$ is given by
$$
\nabla \widehat{u_{h}}(\mathbf{x};\tau) = \left(\frac{\sum_{i=1}^{n}y_{i1}\exp\hspace{2pt}c_{i}}{\sum_{i=1}^{n}\exp c_{i}},\frac{\sum_{i=1}^{n}y_{i2}\exp c_{i}}{\sum_{i=1}^{n}\exp c_{i}},\dots,\frac{\sum_{i=1}^{n}y_{id}\exp c_{i}}{\sum_{i=1}^{n}\exp c_{i}}\right)
$$ where $\mathbf{y}_{i}=\left(y_{i1},y_{i2},\dots,y_{id}\right)$ is the $i$-th training latent code i.e. the encoding of the $i$-th training data point in the latent space $\mathcal{Z}$. Observe that, for any fixed $\mathbf{x}$, $$ \|\nabla u_{h}(\mathbf{x})\|_{\infty} = \underset{i=1,2,\dots,d}{\max}\left|y_{mi}\right|$$ where $m \in \left\{1,2,\ldots,n\right\}$ is such that $\mathbf{x}^{T}\mathbf{y}_{i}+h_{i}$ is maximized over all $i \in \left\{1,2,\dots,n\right\}$ when $i=m$. Given the training dataset and after the autoencoder has been trained, this is a non-negative constant independent of $\tau$ (and hence $\epsilon$), say $k_{1}$. Further, we have that $k_{1}\leq \underset{i,j=1,2,\dots,d}{\max}\left|y_{ij}\right|=k$, say.

For any fixed $\mathbf{x}$ and $\epsilon$ (hence fixed $\tau$), we have that $$ \|\nabla \widehat{u_{h}}(\mathbf{x;\tau})\|_{\infty} = \underset{k=1,2,\dots,d}{\max}\left|\frac{\sum_{i=1}^{n}y_{ik}\exp c_{i}}{\sum_{i=1}^{n}\exp c_{i}}\right|$$ We note that, for any k, $\frac{\sum_{i=1}^{n}y_{ik}\exp c_{i}}{\sum_{i=1}^{n}\exp c_{i}}$, is a weighted average of $y_{ik}$'s with non-negative weights, and hence must satisfy$$m_{k}=\underset{i=1,2,\dots,d}{min}y_{ik}\leq \frac{\sum_{i=1}^{n}y_{ik}\exp c_{i}}{\sum_{i=1}^{n}\exp c_{i}}\leq \underset{i=1,2,\dots,d}{\max}y_{ik}=M_{k}$$ Then $ \|\nabla \widehat{u_{h}}(\mathbf{x;\tau})\|_{\infty}$ is less than or equal to maximum over all the $\left|m_{k}\right|$'s and $\left|M_{k}\right|$'s, taken together, say $k_{2}$. Given the training dataset and after the autoencoder has been trained, this is a non-negative constant independent of $\tau$ (and hence $\epsilon$). In particular, we observe that $k_{2}\leq \underset{i,j=1,2,\dots,d}{\max}\left|y_{ij}\right|=k$.Thus we obtain that $$\|\nabla u_{h}(\mathbf{x})\|_{\infty}^{1\slash2}+\|\nabla \widehat{u_{h}}(\mathbf{x})\|_{\infty}^{1\slash2}\leq k_{1}^{1\slash2} + k_{2}^{1\slash2}\leq 2k^{1\slash2}$$

Following Lemma 2 in Section 3 of \citet{Convex}, using the inequality, we have that $$u_{h}(\mathbf{x})\geq \nabla \widehat{u_{h}}(\mathbf{x;\tau})$$ for any $\tau \geq 0$.
From Equation \eqref{error bound}, which is a restatement of Equation (21) of \citet{Convex}, we have that $$\|u_{h}(\mathbf{x})-\widehat{u_{h}}(\mathbf{x})\|_{\infty}=\sup _{\mathbf{x}}\left|u_{h}(\mathbf{x})-\widehat{u_{h}}(\mathbf{x};\tau)\right| \leq \tau \log n$$
Based on the above observations and using Proposition \ref{sec:Proposition for final result}, we have, for any $\mathbf{x} \in \mathcal{X}$,
\begin{equation*}\label{final result}
\|\nabla u_{h}(\mathbf{x})-\nabla \widehat{u_{h}}(\mathbf{x})\|_{\mathrm{L}^{2}} \leq 2 C_{\mathcal{X}}\left(\tau \log n\right)^{1 / 2}\times 2k^{1 / 2}=K \times \left(\log n\right)^{1 / 2} \times \tau^{1 / 2}
\end{equation*}
where $C_{\mathcal{X}}$ depends only on $\mathcal{X}$ and is independent of $\tau$ (and hence $\epsilon$), and $K=4C_{\mathcal{X}}k^{1 / 2}$ is a constant independent of $\tau$ (and hence $\epsilon$).

Substituting $\tau=\frac{\epsilon}{log\hspace{2pt}n}$, we have, for any $\mathbf{x} \in \mathcal{X}$
\begin{equation*}\label{last result}
\|\nabla u_{h}(\mathbf{x})-\nabla \widehat{u_{h}}(\mathbf{x})\|_{\mathrm{L}^{2}} \leq \frac{K}{\left(\log n\right)^{1 / 2}} \times \epsilon^{1 / 2}.
\end{equation*}

\end{proof}

Thus as $\epsilon$ decreases, the $L^{2}$ norm of the difference in gradients of the true Brenier potential and the smoothed approximate Brenier potential i.e. the $L^{2}$ norm of the difference between the true OT map $T=\nabla u_{h}(.)$ and the approximate OT map $\hat{T} = \nabla \hat{u_{h}}(.)$ becomes smaller. This implies that for a sufficiently small value of the error bound $\epsilon$, not only the Brenier potentials, but the OT maps themselves becomes closer to each other. However, the important property of continuity of the smoothed OT map is preserved as long as $\epsilon>0$. The choice of $\mathcal{X}$ in Theorem \ref{subsec:final result} is immaterial as long as it is compact and bounded. In particular, we have assumed the noise distribution $\mu$ to be a $d$-dimensional uniform distribution with support $\left[-1,1\right]^{d}$, and hence the theorem applies to the scenario we are interested in the Convex Smoothed AE-OT model.

\section{Experimental results}\label{sec:Experimental Results}
\label{sec:Experimental results}

To validate that our proposed algorithm works in practice, we conduct a series of experiments. We want to study whether our proposed algorithm is able to deal with the problems of mode collapse and mode mixture, and generate high quality samples closely resembling the observed data, with good generalization power i.e. not reconstructing the training data exactly.

First, we compare the performance of both the AE-OT model and our proposed algorithm on two toy datasets consisting of 2-dimensional data points, 2D-ring and 2D-grid, consisting of observations simulated from a mixture of 8-Gaussian and 25-Gaussian distributions, respectively, following the authors of \citet{AEOT}. Descriptions of these datasets are given in the Appendix along with other relevant details regarding the training of the AE-OT model and our proposed model. These are ideal datasets for testing the relative performance of the two models, since both the datasets are multimodal in nature. Since these are 2-dimensional datasets, it does not make sense to embed the data in a latent space and then decode the latent codes to generate new samples i.e. it is not necessary to use an autoencoder. Instead we are able to generate samples directly in this case.

\subsection{2-D Toy Datasets}

\subsubsection{Convex Smoothed AE-OT algorithm performance for optimal $\epsilon$}

We report the results of applying the Convex Smoothed AE-OT on the 2-dimensional datasets 8Gaussian and 25Gaussian.

\begin{itemize}
    \item For initially testing out ideas, we ran simulations on 2-D examples similar to what we have done for AE-OT, following Section 3.2 of the paper \citet{Convex} using the entropy prox function (since it has a closed form for $\hat{u_{h}}(x)$ and more importantly for its gradient, thus requiring no additional computational expense).
    \item Initial experiments show that the uniform bound on error made in approximating $u_{h}(x)$ by $\hat{u_{h}}(x)$, denoted by $\epsilon$, can be used to specify how closely we want $u_{h}(x)$ to be approximated. There is a very simple dependence of this uniform error bound on the regularization hyperparameter $\tau$ used in the approximation. For the entropy prox function the choice $\tau=\frac{\epsilon}{\log n}$ (n is the number of observed latent vectors) yields a uniform error bound of $\epsilon$.
    \item It is observed that mode collapse does not occur for any value of $\epsilon$, and the $\epsilon$ value is inversely related to the degree of mode mixture in the generated samples.
    \item For sufficiently small choices of $\epsilon$ (in the order of $10^{-4}$ or less), we observe that the generated samples cover all the modes of the observed data i.e. there is no mode collapse, and there is no mode mixture also. For larger values of $\epsilon$, with the lowering of the accuracy of approximation, mode mixture occurs. So we observe that at least for these 2 datasets, the proposed modification of the AE-OT methodology works very well.
    \item We use the proposed procedure of choosing the optimal $\epsilon$ value based on the MMD test, obtaining the \say{best} choice of $\epsilon$ for these 2 datasets to be in the order of $10^{-1}$. Since we observe that mode collapse does not occur for any value of $\epsilon$, we can view the procedure of choosing the optimal $\epsilon$ value as a procedure to decide what constitutes an acceptable level of mode mixture in the generated samples for the dataset at hand.
    
\end{itemize}


We obtained results for both the datasets varying $\epsilon$ from $10^{-6}$ to $100$, incrementing by a factor of 10. In addition, while choosing the optimal $\epsilon$ using the MMD test, we obtain generate samples corresponding to additional $\epsilon$ values as dictated by Algorithm \ref{Alg. eps}.

Results for the optimal choice of $\epsilon$ for the 8-Gaussian and 25-Gaussian datasets are displayed here, while those corresponding to $\epsilon$ values $10^{-6}$, $10^{-5}$, $10^{-3}$ and $10^{-2}$ are given in the Appendix (\ref{Appendix:Convex_AE-OT_2d}).

Since there are $n=256$ training samples, we generate an equal number of samples using the Convex Smooth AE-OT model, and perform the two-sample permutation test for equality of observed sample distribution and generated sample distribution
based on the MMD test statistic, using 1000 permutations in case of both the datasets. We choose the significance level to be $\alpha=0.05$ in each case and declare an $\epsilon$ value as optimal if it is within $\alpha+\pm \delta$ where we set $\delta=0.01$.

Following this procedure, the optimal choice of $\epsilon$ for the 8-Gaussian dataset based on the two sample MMD test is obtained to be 0.6, while that for the 25-Gaussian dataset is obtained to be 0.8. The corresponding p-values based on the permutation tests were 0.054 in both cases.

The generated samples together with the observed data are as follows:

\begin{figure}[H]
\begin{subfigure}{0.8\textwidth}
\includegraphics[width=\textwidth]{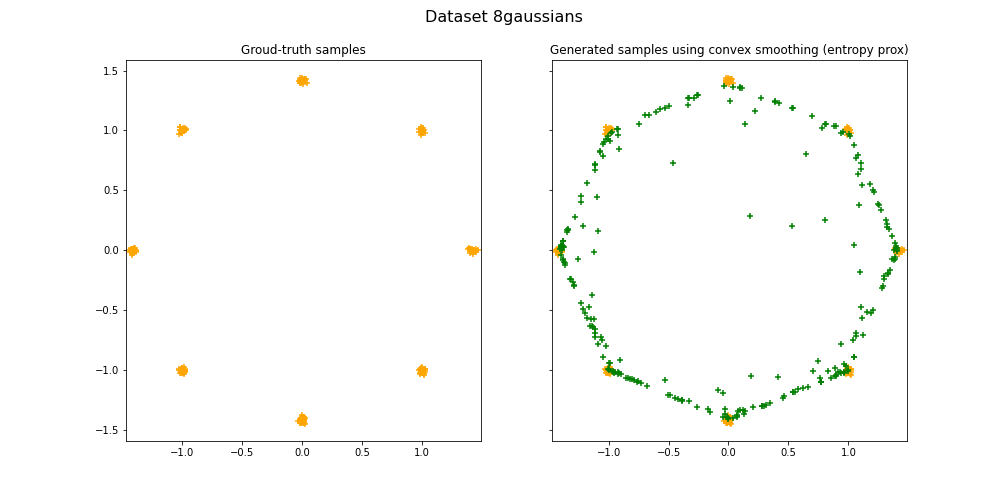}
\caption{8Gaussians Dataset $\epsilon=0.6$}
\label{8gaussianeps0.6}
\end{subfigure}
\begin{subfigure}{0.8\textwidth}
\includegraphics[width=\textwidth]{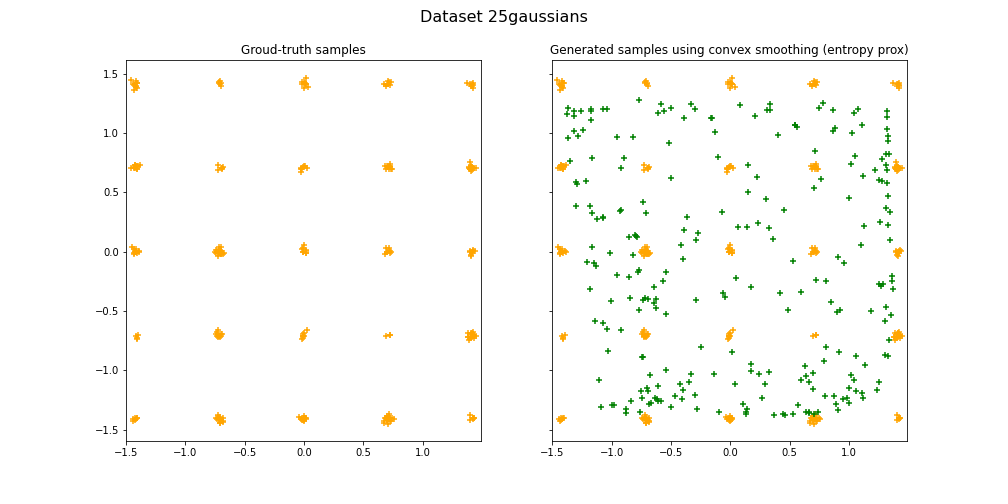}
\caption{25Gaussians Dataset $\epsilon=0.8$}
\label{25gaussianeps0.8}
\end{subfigure}
\end{figure}

\textbf{Comments:} We observe that the generated samples cover all the modes of the observed data and hence the phenomenon of mode collapse is mitigated here. Further, all the generated samples are mixtures of two nearest modes and falls close to the approximate manifold defined by the observed data.

\subsubsection{AE-OT results for varying $\theta$}

To appreciate the efficacy of the Convex Smoothed AE-OT model, it is required to compare its performance with that of the AE-OT model itself on the 8-Gaussian and 25-Gaussian datasets. We provide the results obtained using the AE-OT model along with the relevant discussions in the Appendix (\ref{Appendix:AE-OT_2d}).

\section{Conclusion}\label{sec:Conclusion}

As seen in the Experimental results section (Section \ref{sec:Experimental results}), our proposed generative model - Convex Smoothed AE-OT, produces affirmative results. We improve upon the original AE-OT model (\citet{AEOT}) with regards to its sample generation algorithm, while ensuring that mode collapse is absent and mode mixture is present only upto an allowable level in the generated samples. In addition to empirically validating the efficacy of the proposed model, we provide a theoretical justification for the approximated OT map $\hat{T}$ for being close to the true OT map $T$. Our current efforts are aimed at applying the Convex Smoothed AE-OT model to benchmark Image datasets and evaluating its performance.  


\section*{Acknowledgements}
The author is extremely grateful to Prof. Bodhisattva Sen for guiding her in the development of this paper.

\bibliographystyle{chicago}
\bibliography{Conv-Smooth-AE-OT}

\section*{Appendix}

\subsection*{Experimental results obtained using Convex Smoothed AE-OT model on 2D Toy Datasets}\makeatletter\def\@currentlabel{Convex Smoothed AE-OT results}\makeatother
\label{Appendix:Convex_AE-OT_2d}

We obtained results for both the datasets varying $\epsilon$ from $10^{-6}$ to $10^{-2}$, incrementing by a factor of 10. We display results for all choices except for $10^{-4}$ to save space.

The results corresponding to the 8-Gaussian dataset is displayed on the left and those corresponding to the 25-Gaussian dataset is displayed on the right. The left subplot in each diagram corresponds to the given dataset of 256 data points (marked in orange), while the right subplot in the diagram contains the generated samples (marked in green) superimposed over the original data points (marked in orange).

\subsubsection*{$\epsilon=10^{-6}$ (no mode collapse/mode mixture but exact reconstruction):}

\begin{figure}[H]
\begin{subfigure}{0.45\textwidth}
\includegraphics[width=\textwidth]{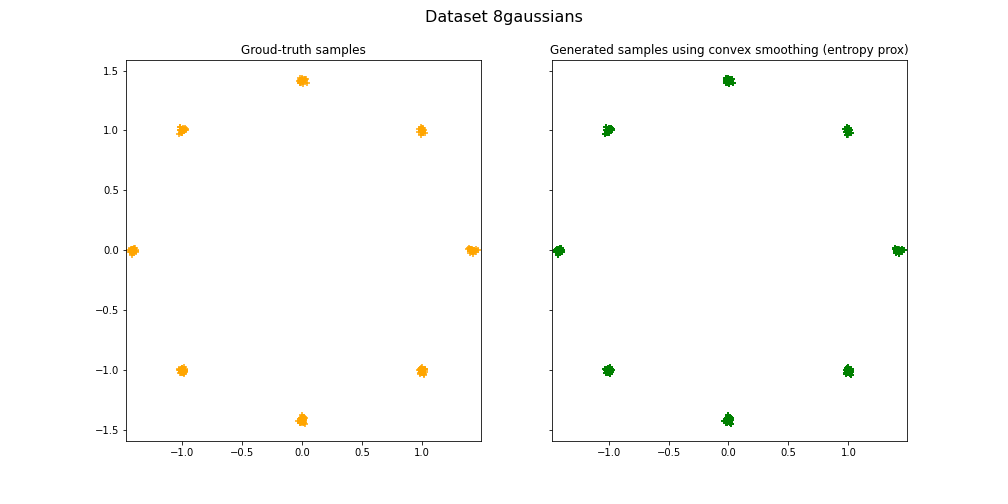}
\caption{8-Gaussians Dataset $\epsilon=10^{-6}$}
\label{8gaussianeps0.000001}
\end{subfigure}
\begin{subfigure}{0.45\textwidth}
\includegraphics[width=\textwidth]{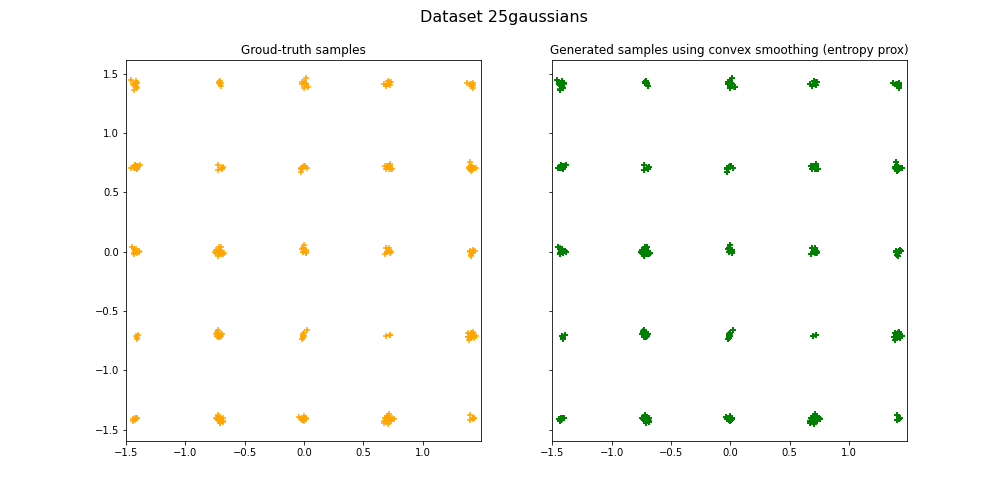}
\caption{25-Gaussians Dataset $\epsilon=10^{-6}$}
\label{25gaussianeps0.000001}
\end{subfigure}
\end{figure}

\subsubsection*{$\epsilon=10^{-5}$:}

\begin{figure}[H]
\begin{subfigure}{0.45\textwidth}
\includegraphics[width=\textwidth]{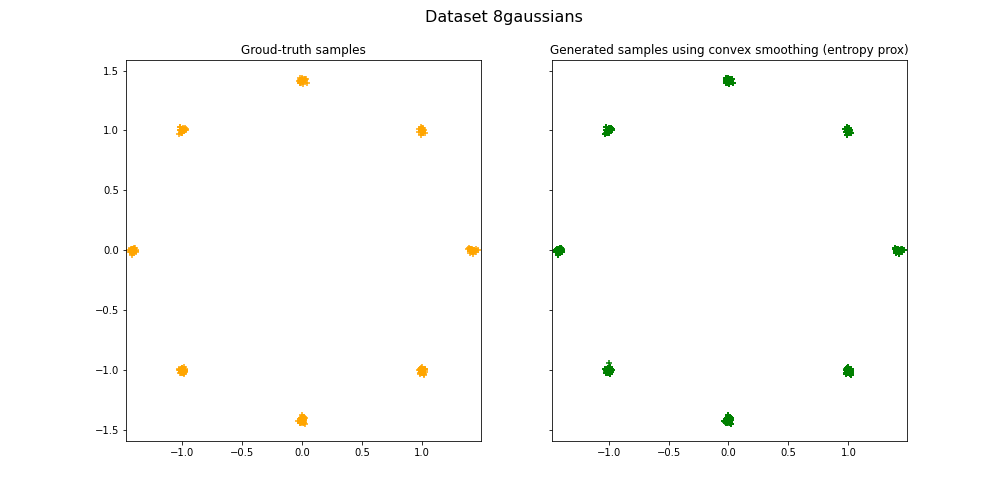}
\caption{8-Gaussians Dataset $\epsilon=10^{-5}$}
\label{8gaussianeps0.00001}
\end{subfigure}
\begin{subfigure}{0.45\textwidth}
\includegraphics[width=\textwidth]{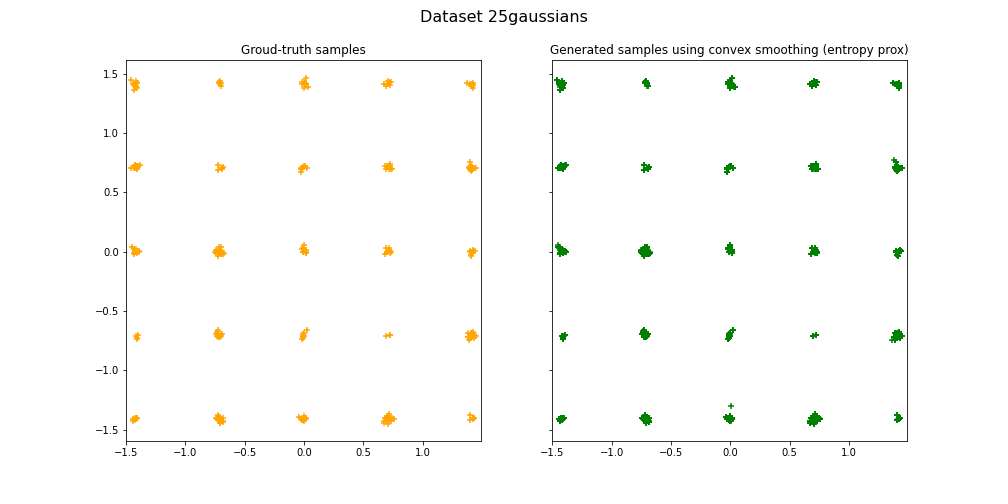}
\caption{25-Gaussians Dataset $\epsilon=10^{-5}$}
\label{25gaussianeps0.00001}
\end{subfigure}
\end{figure}

\subsubsection*{$\epsilon=10^{-3}$:}

\begin{figure}[H]
\begin{subfigure}{0.45\textwidth}
\includegraphics[width=\textwidth]{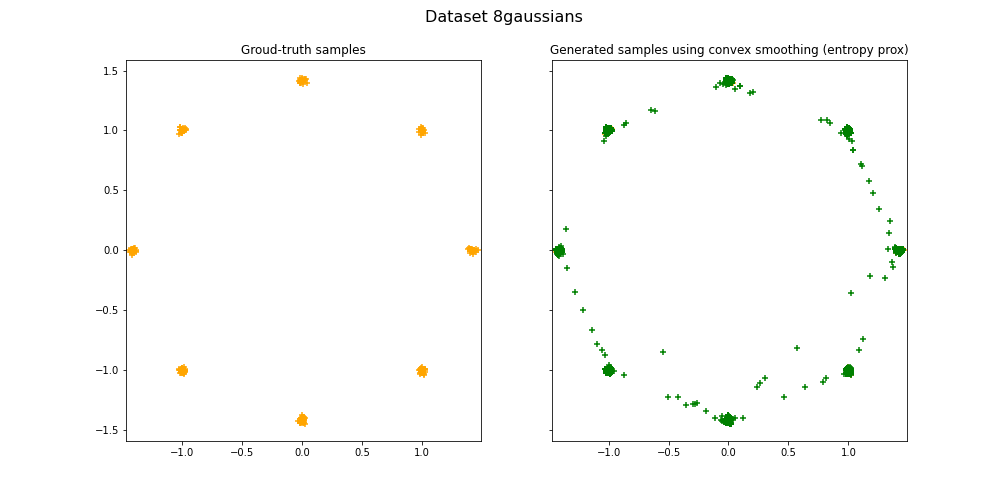}
\caption{8-Gaussians Dataset $\epsilon=10^{-3}$}
\label{8gaussianeps0.001}
\end{subfigure}
\begin{subfigure}{0.45\textwidth}
\includegraphics[width=\textwidth]{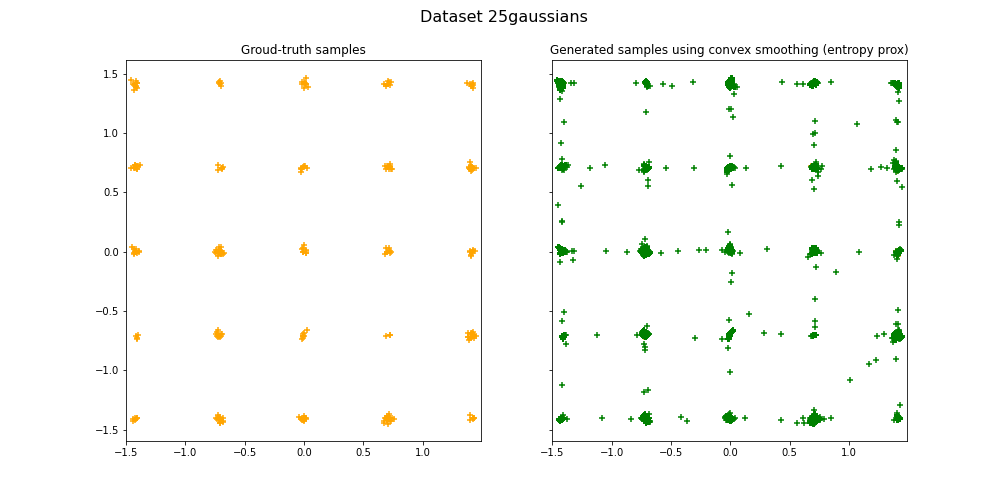}
\caption{25-Gaussians Dataset $\epsilon=10^{-3}$}
\label{25gaussianeps0.001}
\end{subfigure}
\end{figure}

\subsubsection*{$\epsilon=10^{-2}$:}

\begin{figure}[H]
\begin{subfigure}{0.45\textwidth}
\includegraphics[width=\textwidth]{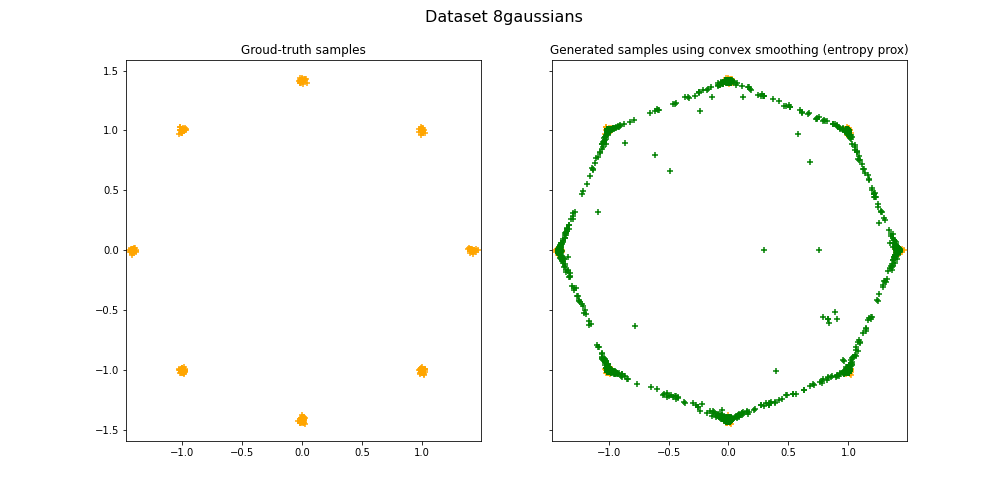}
\caption{8-Gaussians Dataset $\epsilon=10^{-2}$}
\label{8gaussianeps0.01}
\end{subfigure}
\begin{subfigure}{0.45\textwidth}
\includegraphics[width=\textwidth]{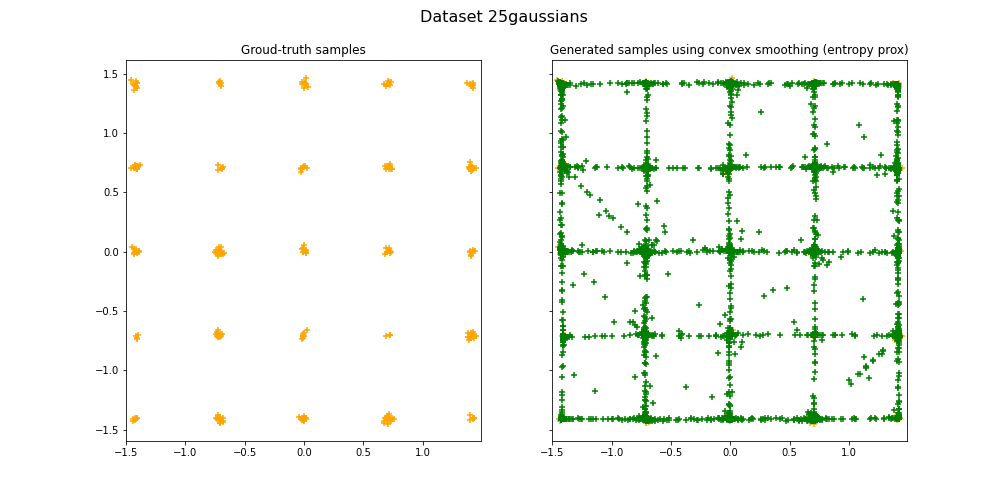}
\caption{25-Gaussians Dataset $\epsilon=10^{-2}$}
\label{25gaussianeps0.01}
\end{subfigure}
\end{figure}

\textbf{Comments:} As we are decreasing $\epsilon$, the accuracy of approximation is increasing and we can visually observe the increase in quality of the generated samples, with mode mixture vanishing for smaller values of $\epsilon$. For no choice of $\epsilon$ do we observe the phenomenon of mode collapse.

\subsection*{Experimental Results obtained using original AE-OT model on 2D Toy Datasets}\makeatletter\def\@currentlabel{AE-OT results}\makeatother
\label{Appendix:AE-OT_2d}

We report in detail about the best performing hyperparameter choice, and visually show the results for other choices of hyperparameters. The important hyperparameter that determines the efficacy of AE-OT in mitigating the mode-collapse/mixture problem is the threshold $\hat{\theta}$ (we try to estimate a good value for $\theta$) set for the dihedral angle between the hyperplanes of the Brenier potential function $u_{h}$ for generating samples using the AE-OT model.

For the 8-Gaussian and 25-Gaussian datasets, we found the best learning rate $\alpha$ for the Adam algorithm used to minimize the convex energy function E to be about 0.0002 and 0.001 respectively. For the 8 Gaussian dataset and the chosen setting of the hyperparameters, the algorithm converged in about 6500 iterations under 6 minutes on the Google Colab GPU Platform. For the 25 Gaussian dataset and the chosen setting of the hyperparameters, the algorithm converged in about 4000 iterations under 4 minutes on the Google Colab GPU Platform. We found the best performing value of $\hat{\theta}$ to be 0.4 and 0.2, respectively for the 2 datasets. For the 8-Gaussian dataset, we tested for a few random values of $\hat{\theta}$ such as 0.001,0.01, 0.2, 0.8 and 1 to test the sensitivity of the results with respect to $\hat{\theta}$. We found that mode collapse happens at $\hat{\theta}=0.001$ and mode mixture happens when $\hat{\theta} \geq 0.8$, while all the modes are covered and no mode mixture happens when $0.01 \leq \hat{\theta} \leq 0.4$. For the 25-Gaussian dataset, we tested for a few random values of $\hat{\theta}$ such as 0.005,0.01,0.1, 0.2, 0.7 and 1 to test the sensitivity of the results with respect to $\hat{\theta}$. We found that mode collapse happens at $\hat{\theta}=0.005$ (severe) and $\hat{\theta}=0.01$ (moderate). Mode mixture happens when $\hat{\theta} \geq 0.7$, while all the modes are covered and no mode mixture happens when $0.1 \leq \hat{\theta} \leq 0.2$. 

We display the results obtained corresponding to the best performing choice of $\hat{\theta}$, close to best choice and choices leading to mode collapse or mode mixture. The arrangement and description of the plots are the same as in the previous section.

\subsubsection*{Best performance:}

\begin{figure}[H]
\begin{subfigure}{0.45\textwidth}
\includegraphics[width=\textwidth]{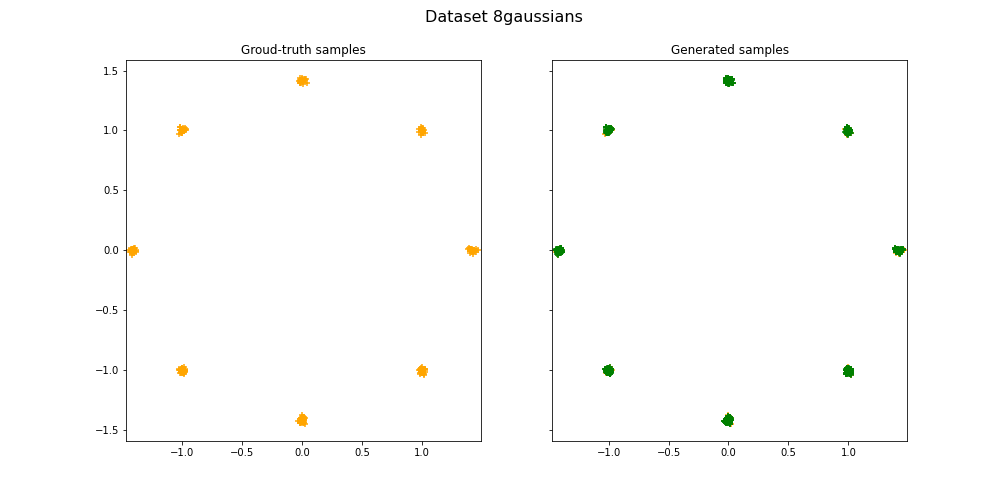}
\caption{8-Gaussians Dataset $\hat{\theta}=0.4$}
\label{8gaussian0.4}
\end{subfigure}
\begin{subfigure}{0.45\textwidth}
\includegraphics[width=\textwidth]{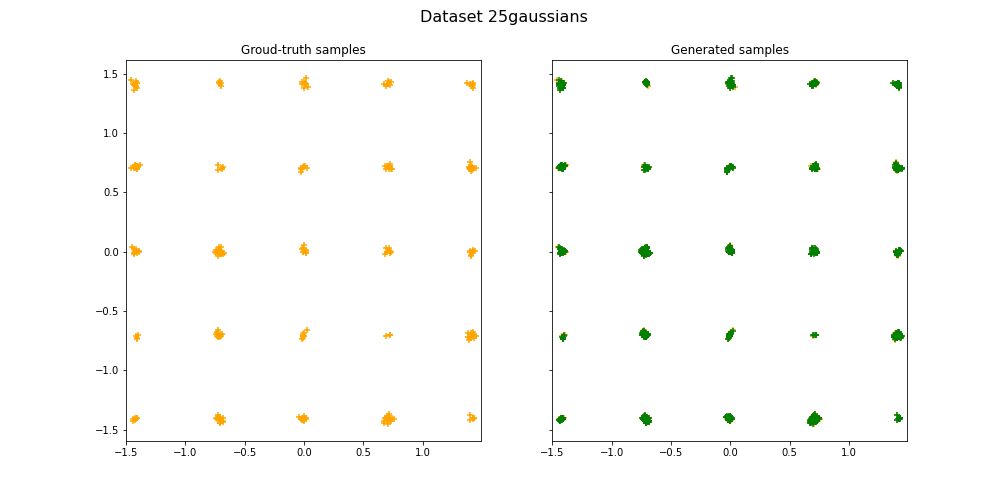}
\caption{25-Gaussians Dataset $\hat{\theta}=0.2$}
\label{25gaussian0.2}
\end{subfigure}
\end{figure}

\subsubsection*{Close to best:}

\begin{figure}[H]
\begin{subfigure}{0.45\textwidth}
\includegraphics[width=\textwidth]{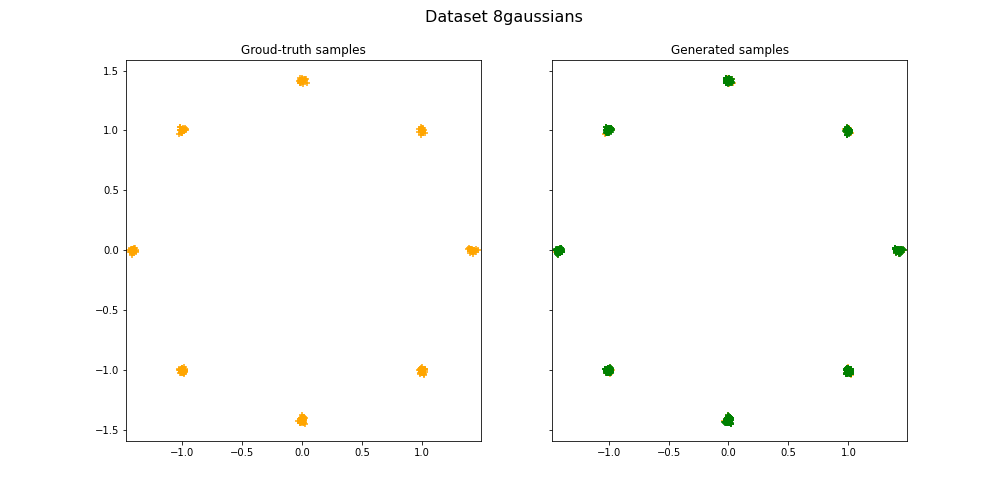}
\caption{8-Gaussians Dataset $\hat{\theta}=0.1$}
\label{8gaussian0.1}
\end{subfigure}
\begin{subfigure}{0.45\textwidth}
\includegraphics[width=\textwidth]{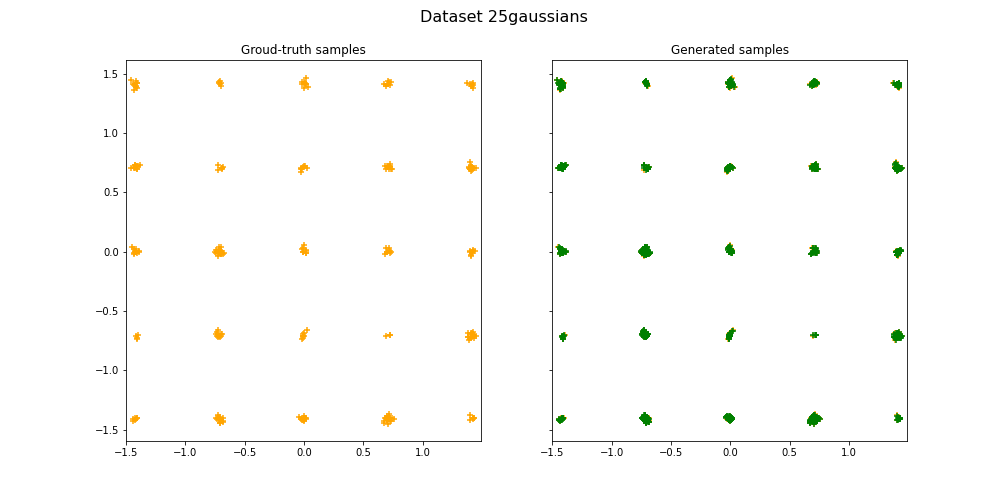}
\caption{25-Gaussians Dataset $\hat{\theta}=0.1$}
\label{25gaussian0.1}
\end{subfigure}
\end{figure}

\subsubsection*{Mode collapse:}

\begin{figure}[H]
\begin{subfigure}{0.45\textwidth}
\includegraphics[width=\textwidth]{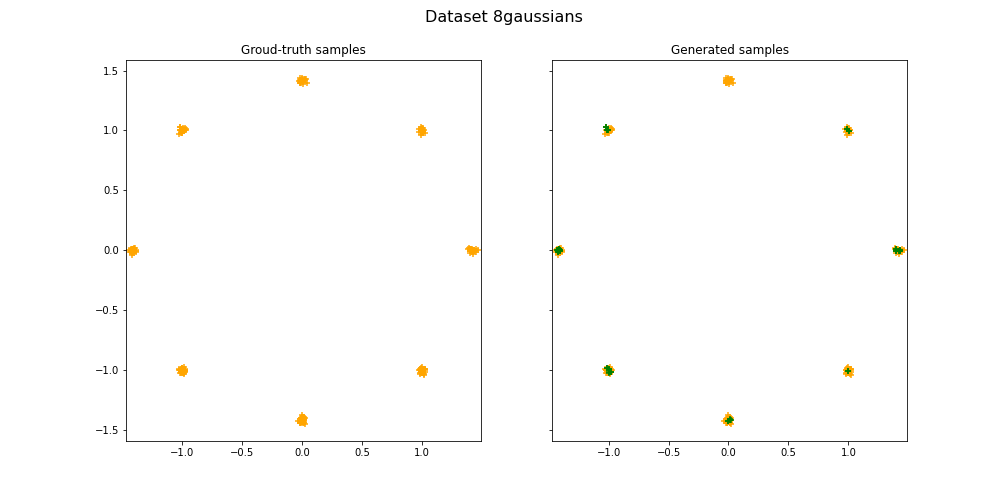}
\caption{8-Gaussians Dataset $\hat{\theta}=0.001$}
\label{8gaussian0.001}
\end{subfigure}
\begin{subfigure}{0.45\textwidth}
\includegraphics[width=\textwidth]{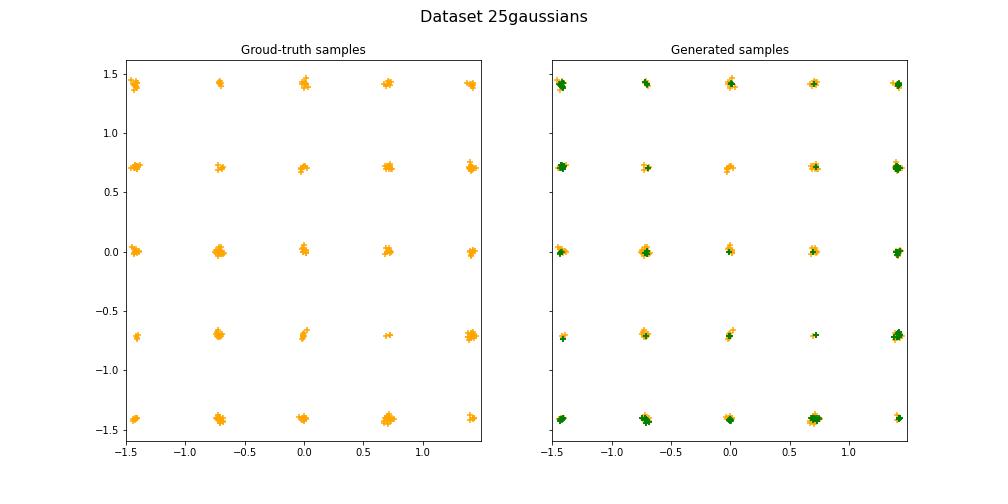}
\caption{25-Gaussians Dataset $\hat{\theta}=0.005$}
\label{25gaussian0.005}
\end{subfigure}
\end{figure}

\subsubsection*{Mode mixture:}

\begin{figure}[H]
\begin{subfigure}{0.45\textwidth}
\includegraphics[width=\textwidth]{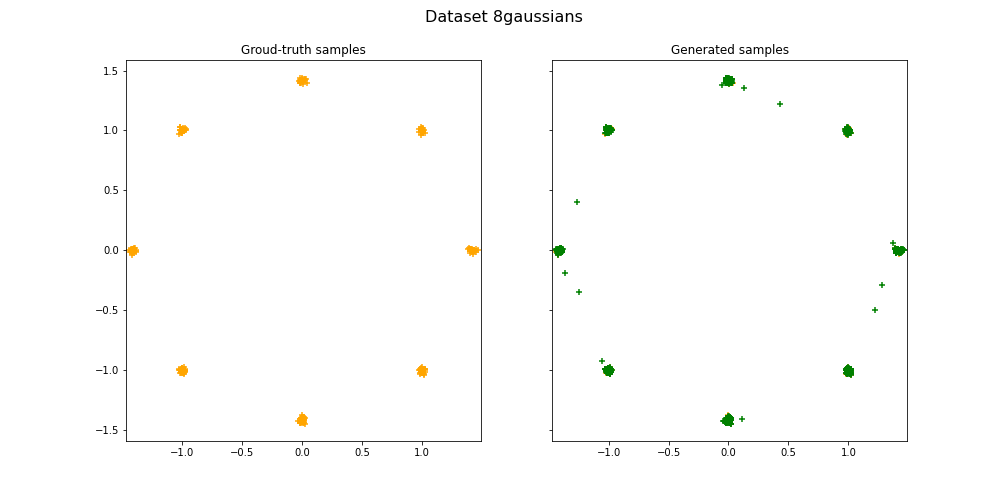}
\caption{8-Gaussians Dataset $\hat{\theta}=1$}
\label{8gaussian1}
\end{subfigure}
\begin{subfigure}{0.45\textwidth}
\includegraphics[width=\textwidth]{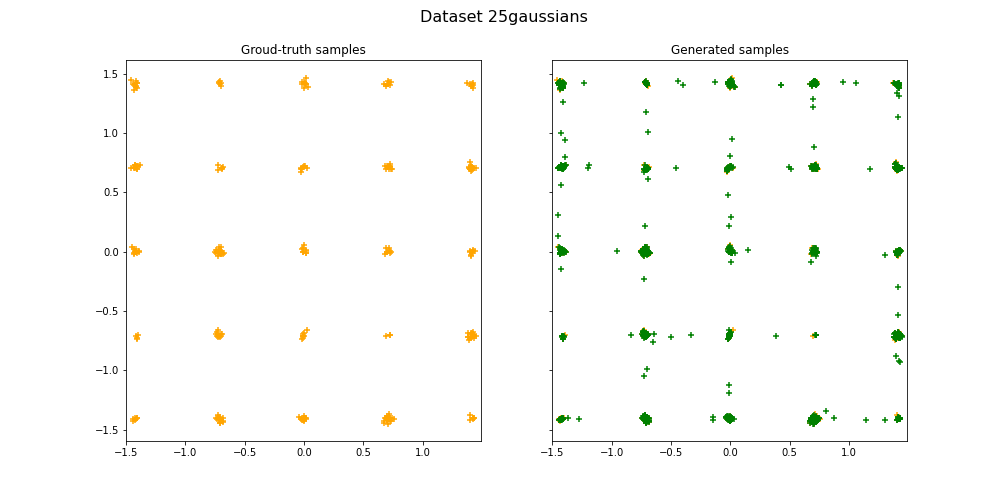}
\caption{25-Gaussians Dataset $\hat{\theta}=0.7$}
\label{25gaussian0.7}
\end{subfigure}
\end{figure}


\textbf{Comments:} For both the datasets, as the threshold parameter $\hat{\theta}$ is increased, we observe greater degree of mode mixture and lower quality of generated samples. Similarly for too low a value of $\hat{\theta}$, we observe mode collapse.

\end{document}